\documentclass[onefignum,onetabnum]{siamonline250211}

\usepackage{lipsum}
\usepackage{amsfonts}
\usepackage{graphicx}
\usepackage{epstopdf}
\usepackage{algorithmic}
\ifpdf
	\DeclareGraphicsExtensions{.eps,.pdf,.png,.jpg}
\else
	\DeclareGraphicsExtensions{.eps}
\fi

\usepackage{enumitem}
\setlist[enumerate]{leftmargin=.5in}
\setlist[itemize]{leftmargin=.5in}

\newsiamremark{remark}{Remark}
\newsiamremark{hypothesis}{Hypothesis}
\crefname{hypothesis}{Hypothesis}{Hypotheses}
\newsiamthm{claim}{Claim}
\newsiamremark{fact}{Fact}
\crefname{fact}{Fact}{Facts}

\headers{Efficient Sparsification of Simplicial Complexes}{A. Savostianov, M.T. Schaub, N. Guglielmi, and F.Tudisco}

\title{Efficient Sparsification of Simplicial Complexes via Local Densities of States\thanks{Submitted to the editors DATE.
	}}

\author{Anton Savostianov\thanks{Computational Network Science, RWTH, Aachen
		(\email{savostianov@cs.rwth-aachen.de}, \url{http://antsav.me}).}
	\and Michael T. Schaub\thanks{Computational Network Science, RWTH, Aachen
		(\email{schaub@cs.rwth-aachen.de}).}
	\and Nicola Guglielmi\thanks{Gran Sasso Science Institute, L'Aquila, Italy
		(\email{nicola.guglielmi@gssi.it}).}
	\and Francesco Tudisco\footnotemark[4] \thanks{University of Edinburgh, Edinburgh, UK
		(\email{f.tudisco@ed.ac.uk}).}}

\usepackage{amsopn}
\DeclareMathOperator{\diag}{diag}

\newcommand*{\mc}[1]{\mathcal{#1}}
\renewcommand*{\bar}[1]{\overline{#1}}
\renewcommand*{\b}[1]{\mathbf{#1}}
\newcommand*\eps{\varepsilon}
\newcommand*{\V}[1]{ \mc V_{#1}(\mc K)}
\usepackage{dsfont}
\newcommand*{\ds}[1]{\mathds{#1}}
\newcommand*{\Lu}[1]{L_{#1}^{\uparrow}}

\newcommand*{\Ld}[1]{L_{#1}^{\downarrow}}

\newcommand*{\wh}[1]{\widehat{#1}}
\renewcommand*{\bar}[1]{ \overline{#1} }

\definecolor{rwth-blue}{HTML}{00549F}
\definecolor{rwth-blue-75}{HTML}{407FB7}
\definecolor{rwth-blue-50}{HTML}{8EBAE5}
\definecolor{rwth-blue-25}{HTML}{C7DDF2}
\definecolor{rwth-blue-10}{HTML}{E8F1FA}
\definecolor{rwth-black}{HTML}{000000}
\definecolor{rwth-black-75}{HTML}{646567}
\definecolor{rwth-black-50}{HTML}{9C9E9F}
\definecolor{rwth-black-25}{HTML}{CFD1D2}
\definecolor{rwth-black-10}{HTML}{ECEDED}
\definecolor{rwth-magenta}{HTML}{E30066}
\definecolor{rwth-magenta-75}{HTML}{E96088}
\definecolor{rwth-magenta-50}{HTML}{F19EB1}
\definecolor{rwth-magenta-25}{HTML}{F9D2DA}
\definecolor{rwth-magenta-10}{HTML}{FDEEF0}
\definecolor{rwth-yellow}{HTML}{FFED00}
\definecolor{rwth-yellow-75}{HTML}{FFF055}
\definecolor{rwth-yellow-50}{HTML}{FFF59B}
\definecolor{rwth-yellow-25}{HTML}{FFFAD1}
\definecolor{rwth-yellow-10}{HTML}{FFFDEE}
\definecolor{rwth-petrol}{HTML}{006165}
\definecolor{rwth-petrol-75}{HTML}{2D7F83}
\definecolor{rwth-petrol-50}{HTML}{7DA4A7}
\definecolor{rwth-petrol-25}{HTML}{BFD0D1}
\definecolor{rwth-petrol-10}{HTML}{E6ECEC}
\definecolor{rwth-turquoise}{HTML}{0098A1}
\definecolor{rwth-turquoise-75}{HTML}{00B1B7}
\definecolor{rwth-turquoise-50}{HTML}{89CCCF}
\definecolor{rwth-turquoise-25}{HTML}{CAE7E7}
\definecolor{rwth-turquoise-10}{HTML}{EBF6F6}
\definecolor{rwth-green}{HTML}{57AB27}
\definecolor{rwth-green-75}{HTML}{8DC060}
\definecolor{rwth-green-50}{HTML}{B8D698}
\definecolor{rwth-green-25}{HTML}{DDEBCE}
\definecolor{rwth-green-10}{HTML}{F2F7EC}
\definecolor{rwth-maigrun}{HTML}{BDCD00}
\definecolor{rwth-maigrun-75}{HTML}{D0D95C}
\definecolor{rwth-maigrun-50}{HTML}{E0E69A}
\definecolor{rwth-maigrun-25}{HTML}{F0F3D0}
\definecolor{rwth-maigrun-10}{HTML}{F9FAED}
\definecolor{rwth-orange}{HTML}{F6A800}
\definecolor{rwth-orange-75}{HTML}{FABE50}
\definecolor{rwth-orange-50}{HTML}{FDD48F}
\definecolor{rwth-orange-25}{HTML}{FEEAC9}
\definecolor{rwth-orange-10}{HTML}{FFF7EA}
\definecolor{rwth-red}{HTML}{CC071E}
\definecolor{rwth-red-75}{HTML}{D85C41}
\definecolor{rwth-red-50}{HTML}{E69679}
\definecolor{rwth-red-25}{HTML}{F3CDBB}
\definecolor{rwth-red-10}{HTML}{FAEBE3}
\definecolor{rwth-bordeaux}{HTML}{A11035}
\definecolor{rwth-bordeaux-75}{HTML}{B65256}
\definecolor{rwth-bordeaux-50}{HTML}{CD8B87}
\definecolor{rwth-bordeaux-25}{HTML}{E5C5C0}
\definecolor{rwth-bordeaux-10}{HTML}{F5E8E5}
\definecolor{rwth-violet}{HTML}{612158}
\definecolor{rwth-violet-75}{HTML}{834E75}
\definecolor{rwth-violet-50}{HTML}{A8859E}
\definecolor{rwth-violet-25}{HTML}{D2C0CD}
\definecolor{rwth-violet-10}{HTML}{EDE5EA}
\definecolor{rwth-purple}{HTML}{7A6FAC}
\definecolor{rwth-purple-75}{HTML}{9B91C1}
\definecolor{rwth-purple-50}{HTML}{BCB5D7}
\definecolor{rwth-purple-25}{HTML}{DEDAEB}
\definecolor{rwth-purple-10}{HTML}{F2F0F7}

\usepackage{tikz}
\usepackage{pgfplots}

\usetikzlibrary{arrows.meta}
\usetikzlibrary{backgrounds}
\usepgfplotslibrary{patchplots}
\usepgfplotslibrary{fillbetween}
\pgfplotsset{%
	layers/standard/.define layer set={%
			background,axis background,axis grid,axis ticks,axis lines,axis tick labels,pre main,main,axis descriptions,axis foreground%
		}{
			grid style={/pgfplots/on layer=axis grid},%
			tick style={/pgfplots/on layer=axis ticks},%
			axis line style={/pgfplots/on layer=axis lines},%
			label style={/pgfplots/on layer=axis descriptions},%
			legend style={/pgfplots/on layer=axis descriptions},%
			title style={/pgfplots/on layer=axis descriptions},%
			colorbar style={/pgfplots/on layer=axis descriptions},%
			ticklabel style={/pgfplots/on layer=axis tick labels},%
			axis background@ style={/pgfplots/on layer=axis background},%
			3d box foreground style={/pgfplots/on layer=axis foreground},%
		},
}

\usetikzlibrary{decorations.pathreplacing,decorations.markings}
\tikzset{
	on each segment/.style={
			decorate,
			decoration={
					show path construction,
					moveto code={},
					lineto code={
							\path [#1]
							(\tikzinputsegmentfirst) -- (\tikzinputsegmentlast);
						},
					curveto code={
							\path [#1] (\tikzinputsegmentfirst)
							.. controls
							(\tikzinputsegmentsupporta) and (\tikzinputsegmentsupportb)
							..
							(\tikzinputsegmentlast);
						},
					closepath code={
							\path [#1]
							(\tikzinputsegmentfirst) -- (\tikzinputsegmentlast);
						},
				},
		},
	mid arrow/.style={postaction={decorate,decoration={
							markings,
							mark=at position .5 with {\arrow[#1]{stealth}}
						}}},
}

\usepackage{tikz-network}
\usetikzlibrary{patterns}
\newcommand{\AxisRotator}[1][rotate=0]{%
	\tikz [x=0.15cm,y=0.15cm,line width=.2ex,-stealth,#1] \draw (0,0) arc (-150:150:1 and 1);%
}

\DeclareMathOperator{\im}{im}
\usepackage{bm}
\newtheorem{problem}{Problem}

\usepackage{mathtools}

\ifpdf
\hypersetup{
  pdftitle={Efficient Sparsification of Simplicial Complexes via Local Densities of States},
  pdfauthor={A. Savostianov, M.T. Schaub, N. Guglielmi, and F.Tudisco}
}
\fi

\begin{document}

\maketitle

\begin{abstract}
	Simplicial complexes (SCs) have become a popular abstraction for analyzing complex data using tools from topological data analysis or topological signal processing.
	However, the analysis of many real-world datasets often leads to dense SCs, with many higher-order simplicies, which results in prohibitive computational requirements in terms of time and memory consumption.
	The sparsification of such complexes is thus of broad interest, i.e., the approximation of an original SC with a sparser surrogate SC (with typically only a log-linear number of simplices) that maintains the spectrum of the original SC as closely as possible.
	In this work, we develop a novel method for a probabilistic sparsification of SCs that uses so-called local densities of states.
	Using this local densities of states, we can efficiently approximate so-called generalized effective resistance of each simplex, which is proportional to the required sampling probability for the sparsification of the SC.
	To avoid degenerate structures in the spectrum of the corresponding Hodge Laplacian operators, we suggest a ``kernel-ignoring'' decomposition to approximate the sampling probability.
	Additionally, we utilize certain error estimates to characterize the asymptotic algorithmic complexity of the developed method.
	We demonstrate the performance of our framework on a family of Vietoris--Rips filtered simplicial complexes.
\end{abstract}



\section{Introduction}

Graphs have become a virtually ubiquitous modeling tool for complex systems and relational data.
They encode structural information in various machine learning tasks, and are often the subject of analysis in their own right, e.g., in the context of link prediction, node importance ranking, or label propagation, with numerous applications across disciplines.
However, graphs are inherently restricted to model pairwise interactions between entities~\cite{benson2016higher,bick2023higher,battiston2020networks}, while many natural systems include polyadic interactions (occurring between multiple entities), such as chemical reactions, co-authorship networks, and social connections.
To address this limitation, higher-order models of relational data have gained popularity recently, including hypergraphs, motifs, cell and simplicial complexes~\cite{battiston2020networks,bick2023higher,schaub2021signal}.
Accounting for such higher-order effects can have profound effects, as polyadic interactions can, e.g., promote synchronization~\cite{gambuzza2021stability}, alter and enhance label propagation methods and critical node identification~\cite{tudisco2021nonlinear,prokopchik2022nonlinear}, or support higher-order random walks and trajectory classification~\cite{schaub2019random,grande2024topological}.

However, as system size increases, the number of (possible) higher-order interactions scales combinatorially, which often leads to (often dense) higher-order models that are computationally prohibitive to analyze.
It is  thus natural to pose the question, if such higher-order networks can be sparsified: for a given high-order model \( \mc K \), can we find a model \( \mc L \) with similar key properties (such as similar topology or comparable rates of information propagation), but with significantly fewer \( k\)-order interactions?

\paragraph{Related Work}
The notion of sparsification has been extensively discussed for graphs with numerous sparsification algorithms being proposed over the years.
These methods aim to reduce the number of edges and preserve specific properties of the original graphs, such as cut costs~\cite{benczur1996approximating,ahn2012graph}, clustering~\cite{satuluri2011local}, or classification scores~\cite{li2022graph}.
One prominent approach introduced by Spielman and Srivastava is the concept of spectral sparsification, \cite{spielman2008graph,spielman2011spectral}, which preserves the spectrum of the corresponding graph Laplacian operator \( L_0 \) by sampling edges based on their generalized effective resistance~\cite{tetali1991random}.
Sparsification is also closely related to the Lottery Ticket theorem for (graph) neural networks, \cite{chen2021unified}, which posits that dense, randomly initialized neural networks contain sparse subnetworks that can be trained in isolation to achieve performance comparable to the original network.

The spectral sparsification approach has later been extended to the case of higher-order models such as hypergraphs, \cite{soma2019spectral}, and simplicial complexes, \cite{osting2017spectral}.
The latter is of specific interest if we are interested in topological properties, since SCs are generalizations of classical graph models that include higher-order relations represented by nodal simplices, aligning with the intrinsic topology of the data, \cite{quillen2006homotopical, munkres2018elements}.
Specifically, the corresponding higher-order Hodge Laplacian operators \( L_k \) define homology groups which describe \( k \)-dimensional ``holes'' (connected components, loops, voids, etc.) in the complex, relating homology to the elements of the kernel of \( L_k \), \cite{Lim15}.
Interestingly, spectral information of \( L_k \) can be used to determine topological stability of the simplicial complex, \cite{guglielmi2023quantifying}, and can be exploited for spectral clustering \cite{ebli2019notion,grande2024node,grande2024topological}, which makes the spectral sparsification of SCs particularly relevant.
These topological and spectral properties can also be employed in topological signal processing, since the Hodge Laplacians  \( L_k \) can act as natural structural shift operator for signals defined on complexes\cite{barbarossa2020topological,schaub2021signal,schaub2022signal}, and thus is also instrumental for defining neural networks defined on simplicial complexes \cite{ebli2020simplicial,roddenberry2021principled,yang2022simplicial}.
Consequently, given the importance of the spectral information of \( L_k \) for the topological features of the complex \( \mc K \), a natural goal is to develop sparsificationo schemes that preserve spectral properties of \( L_k \).

\paragraph{Overview of Approach and Main Results}
In this work, we consider the efficient spectral sparsification of simplicial complexes at the level of simplices of order \( k \), i.e., the reduction of the number of \(( k + 1 )\)-simplices with respect to the number of \( k\)-simpices while maintaining
the spectrum or specific spectral properties of the original operator \( L_k \) (note that the spectral properties of all lower order Laplacians are unaffected by the sparsification).
To be more specific, each operator  \(L_k\) is composed of down- and up-Laplacian terms, \(L_k = \Ld k + \Lu k\).
Since \( \Lu k \) describes up-adjacency between simplices of orders \( k \) and \( k+1\), we need to control its spectrum to create a good sparsifier.
Central results from spectral sparsification, \cite{spielman2008graph, spielman2011spectral, osting2017spectral}, guarantee the existence of such a sparsifier \( \mc L \) for any simplicial complex \( \mc K \) with \( m_{k+1}(\mathcal L ) = \mathcal O ( m_{k} (\mathcal K) \log m_{k}(\mathcal K) ) \), where \( m_k(\mathcal K)\) denotes the number of simplices of order $k$ in the simplicial complex \(\mathcal K\).
However, to create such a sparsifier \( \mathcal L \) we need to sample simplices from the original complex \( \mathcal K \) with a probability proportional to the generalized effective resistance vector \( \mathbf r \).
Computing this exactly is inherently computationally prohibitive as it requires access to the entries of the pseudo-inverse of the Laplacian \( \Lu k \).
This computational challenge directly links the graph and simplicial sparsification problems to efficient solvers for linear systems involving the operators $L_k$.
For graphs ($L_0$) several methods have been proposed to efficiently compute (or approximate) the effective resistance, \cite{spielman2014nearly, cohen2014solving, kelner2013simple}.
Unlike the graph case, available approaches for higher-order Laplacians\cite{savostianov2024cholesky,cohen2014solving2,kyng2016approximate} are either uniquely applicable for sparse simplicial complexes or their performance suffers in the dense case.
As a result, spectral sparsification remains a challenge for a large number of simplicial complexes arising in practice.

In this work, we show that effective resistance vector \( \b r \) can be directly computed using functional descriptors of the spectral information known as the network's \emph{local density of states}.
These descriptors, which are well-established across various areas of physics, \cite{weisse2006kernel}, were introduced for network analysis in \cite{dong2019network}, and can be used to encode topological and spectral information as node features \cite{djima2025power}.
To build a sparsification algorithms for SCs based on the local density of states, we first show that the generalized resistance vector \( \b r \) -- and hence our required sampling probability -- can be expressed in terms of the full spectral information of the up-Laplacian \( \Lu k \).
We then reformulate this computation in terms of the local densities of states \( \mu_k(\lambda \mid L_k)\) of the up-Laplacian operator \( L_k\).
One computational challenge that we need to address in this context is that existing kernel polynomial methods to approximate LDoS \cite{weisse2006kernel} fail to converge due to the presence of a large kernel in \( \Lu k \), arising from the specific algebraic structure of the Laplacians $L_k$~\cite{Lim15,schaub2019random}.
(Note that similar behaviour is associated with over-represented motifs in case of graphs, \cite{dong2019network}), which can however be addressed by other means).
To resolve this issue arising for higher-order SCs, we suggest a novel method based on a kernel-ignoring decomposition.
Additionally, we provide error estimates which allow us to derive guarantees for the method's advantageous computational complexity \( \mathcal O \left( \delta^{-3} {m_{k+1}^4}{m_k^{-3}} \right)\), where \( \delta \) controls the approximation error.
The performance of the developed method is illustrated on a family of Vietoris--Rips simplicial complexes \cite{hausmann1994vietoris}, for various density levels and orders of simplices.
We provide a brief schematic overview of the proposed method in \Cref{fig:overview}.

\begin{figure}[hbtp]
	\centering
	\includegraphics[width=1.0\textwidth]{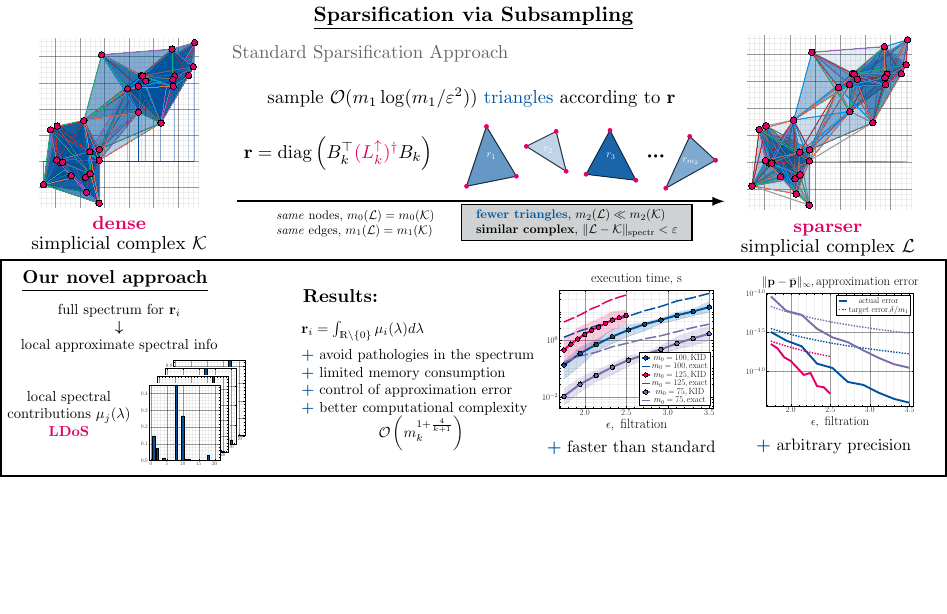}

	\vspace{-75pt}
	\caption{Overview of the proposed method. Top row: the spectral sparsification task and existing sampling approach using generalized effective resistance vector \( \b r \), \cite{spielman2011spectral,osting2017spectral}. Bottom row: proposed method for approximating the vector \( \b r\) using local densities of states \( \mu_j(\lambda)\) and kernel-ignoring decomposition (KID), \Cref{sec:KID}.}
	\label{fig:overview}
\end{figure}

\paragraph{Contributions}
Our main contributions can be summarized as follows:
\begin{itemize}
	\item[\textbf{(i)}] We show that the required sparsification probability for simplicial complexes is related to the local densities of states of a higher-order down-Laplacian {\( \Ld {k+1} \)(whose spectrum is inherited from \( \Lu k \), \cite{guglielmi2023quantifying})}.
	      This measure has previously been defined in terms of the full spectrum of the up-Laplacian; the transition to LDoS enables an efficient approximation.
	\item[\textbf{(ii)}] We develop a novel kernel-ignoring decomposition (KID) for the efficient approximation of LDoS. The suggested method avoids certain spectral structures of the up-Laplacian, which prevent a successful application of preexisting methods.
	\item[\textbf{(iii)}] Finally, we prove that our method outperforms previous approaches in terms of algorithmic complexity, bounding the number of parameters required to obtain a desired approximation error.
\end{itemize}

\paragraph*{Outline} The rest of the paper is organized as follows: \Cref{sec:SC} provides a brief introduction to simplicial complexes.
In \Cref{sec:sparse}, we present our main sparsification result and its connection to the local spectral density of states, which can be used as efficiently computable proxies for spectral information.
\Cref{sec:KID} outlines the proposed novel approach for efficiently computing the sparsifying probability measure.
Finally, numerical experiments are presented in \Cref{sec:benchmark}, followed by concluding remarks in \Cref{sec:discussion}.

\section{Preliminaries}
\label{sec:SC}

\subsection{Notation}
We use \( \sigma(A )\) to denote the spectrum of a symmetric operator \( A \). \( \sigma_+(A )\) denotes the strictly positive part of the spectrum.
We say that two symmetric operators \( A \) and \( B \) are semi-positive ordered \( A \succeq B \) iff \( A - B  \succeq 0 \), meaning \( A - B \) is a symmetric positive semidefinite operator.
Two operators \( A \) and \( B \) are spectrally \( \eps \)-close, \( A \underset{\eps}{\approx} B \) iff
\begin{equation*}
	( 1 - \eps ) B \preceq A \preceq ( 1 + \eps ) B.
\end{equation*}
We use \(\odot \) to denote element-wise matrix multiplication.
Finally, for a finite set \( \mc S \),  \( | \mc S | \) corresponds to its cardinality.

\subsection{Simplicial complexes}

Graphs are typically defined in terms of two sets: a set \( \mc V_0 \) consisting of the nodes of the graph and a set \( \mc V_1 \) consisting of edges between the nodes.
Notably each edge is a relation between two nodes, i.e., a subset of \( \mc V_0 \) of order 2.
\emph{Simplicial complexes} may be considered as a structured generalization, which include higher-order relations between the nodes, i.e., subsets of nodes of cardinality larger than two.
Let us assume that \( \mc V_0 = \{ v_1, v_2, \dots v_{m_0}\}\).
A subset \( \sigma \) of \( \mc V_0 \) with $k+1$ elements is called a \emph{simplex} of order \( k \) (\(k\)-simplex) with its maximal proper subsets of order \( ( k - 1 ) \) known as \emph{faces} of \( \sigma \).
To enable computations, we assume that simplices in \( \V k \) have a fixed lexicographical ordering, which defines an orientation of each simplex.
Then, each subset \( \sigma \subseteq \mc V_0 \) of cardinality \( k + 1\) can be represented as an ordered tuple \( \sigma = [ v_{i_0}, v_{i_1}, \dots v_{i_k} ]\).
\begin{definition}[Simplicial complex, \cite{Lim15}]
	A collection of simplices \( \mc K \) on the node set \( \mc V_0 \) is a \emph{simplicial complex}, if each simplex \( \sigma \) enters \( \mc K \) with all its faces.
	It follows that \( \mc K = \bigcup_{k=0}^{\dim K} \V k \), where \( \V k \) is a set of simplices of order \( k \) and \( \dim K \) is the maximal order of simplices in \( \mc K \).
	We provide a small example of a simplicial complex in \Cref{fig:orientation}.
\end{definition}

Let us denote the number of \( k \)-simplices in \( \mc K\) by \( m_k = | \V k | \). The sparsity of $\mc K$ at the level of \( k\)-simplices is defined by the relation between \( m_k \) and \( m_{k+1}\); we refer to it as the \( k \)-sparsity of the simplicial complex.
In particular, the \( 0 \)-sparsity describes the relation between the number of nodes \( m_0 \) and the number of edges \( m_1 \) and is typically used to define  the sparsity of a graph.
Note that the \( k \)-sparsity is not on its own indicative of the \( ( k + 1 ) \)-sparsity, as both depend on the intrinsic topology of the simplicial complex.
Consequently, if we want to sparsify a simplicial complex, it is natural to consider this problem for a fixed \( k \) rather than attempting to define a unified notion across all simplex orders.

We now formalize the concept of closeness between simplicial complexes.
Following~\cite{spielman2008graph, osting2017spectral}, we use a notion of sparsification defined via the spectral closeness of a family of (higher-order) Laplacian operators \( L_k \) that describe the topology of the underlying simplicial complex.
For this purpose, the operators \( L_k \) are formally defined below.

\subsection{Laplacian Operators and Topology}

In a simplicial complex $\mc K$, each simplex \( \sigma \) is part of \( \mc K \) along with all its faces.
Hence, there exists a map matching it to its boundary formed by its faces.
\begin{figure}[t]
	\centering
	\begin{tikzpicture}
		\begin{scope}[shift={(-0.75, 0)}]
			\draw[fill = rwth-blue, opacity = 0.4] (0,0) -- (1.5,0) -- (0.75, -1.5*3/4) -- cycle;
			\draw[fill = rwth-blue, opacity = 0.6] (0,0) -- (1.5,0) -- (0.75, 1.5*3/4) -- cycle;

			\Vertex[x=0, y=0, label = 1, style={color = rwth-magenta}, fontcolor = white, size = 0.4 ]{v1}
			\Vertex[x=1.5, y=0, label = 3, style={color = rwth-magenta}, fontcolor = white, size = 0.4 ]{v2}
			\Vertex[x=0.75, y=-1.5*3/4, label = 2, style={color = rwth-magenta}, fontcolor = white, size = 0.4 ]{v3}
			\Vertex[x=0.75, y=1.5*3/4, label = 4, style={color = rwth-magenta}, fontcolor = white, size = 0.4 ]{v4}
			\Vertex[x=2.25, y=1.5*3/4, label = 5, style={color = rwth-magenta}, fontcolor = white, size = 0.4 ]{v5}
			\Edge[Direct](v1)(v2)
			\Edge[Direct](v1)(v3)
			\Edge[Direct](v1)(v4)
			\Edge[Direct](v3)(v2)
			\Edge[Direct](v2)(v4)
			\Edge[Direct](v3)(v5)
			\Edge[Direct](v4)(v5)
			\node at (0.75, 1.5*3/4*1/3 ) { \AxisRotator[rotate=0] };
			\node at (0.75, -1.5*3/4*1/3 ) { \AxisRotator[rotate=-60] };
		\end{scope}

		\begin{scope}[shift = {(2.5, -0.5 )}]
			\node[anchor=north west,align=left,] at ( 0, 1.6*3/4 ) { \( \textcolor{rwth-magenta}{\mc V_0(\mc K)} : [1], [2], [3], [4], [5] \) };
			\node[anchor=north west,align=left,] at ( 0, 0.9*3/4 ) { \(
				\begin{aligned}
					\mc V_1(\mc K) : & \; [1, 2], [1, 3], [1, 4],           \\[-5pt]
					                 & \; [ 2, 3], [3, 4 ], [3, 5], [4, 5 ]
				\end{aligned}
				\) };
			\node[anchor=north west,align=left,] at ( 0, -.3*3/4 ) { \(
				\begin{aligned}
					\textcolor{rwth-blue}{\mc V_2(\mc K)} : [ 1, 2, 3], [ 1, 3, 4]
				\end{aligned}
				\) };
		\end{scope}

		\begin{scope}[shift={(-2.75, -3.25)}]
			\Vertex[x=2.5, y=1.5*3/4, label = 1, style={color = rwth-magenta}, fontcolor = white, size = 0.3 ]{t1}
			\Vertex[x=3.5, y=1.5*3/4, label = 2, style={color = rwth-magenta}, fontcolor = white, size = 0.3 ]{t2}
			\Vertex[x=2.5, y=1.1*3/4, label = 1, style={color = rwth-magenta}, fontcolor = white, size = 0.3 ]{t3}
			\Vertex[x=3.5, y=1.1*3/4, label = 3, style={color = rwth-magenta}, fontcolor = white, size = 0.3 ]{t4}
			\Vertex[x=2.5, y=0.7*3/4, label = 1, style={color = rwth-magenta}, fontcolor = white, size = 0.3 ]{t5}
			\Vertex[x=3.5, y=0.7*3/4, label = 4, style={color = rwth-magenta}, fontcolor = white, size = 0.3 ]{t6}
			\Vertex[x=2.5, y=0.3*3/4, label = 2, style={color = rwth-magenta}, fontcolor = white, size = 0.3 ]{t7}
			\Vertex[x=3.5, y=0.3*3/4, label = 3, style={color = rwth-magenta}, fontcolor = white, size = 0.3 ]{t8}
			\Vertex[x=2.5, y=-0.1*3/4, label = 3, style={color = rwth-magenta}, fontcolor = white, size = 0.3 ]{t9}
			\Vertex[x=3.5, y=-0.1*3/4, label = 4, style={color = rwth-magenta}, fontcolor = white, size = 0.3 ]{t10}
			\Vertex[x=2.5, y=-0.5*3/4, label = 3, style={color = rwth-magenta}, fontcolor = white, size = 0.3 ]{t11}
			\Vertex[x=3.5, y=-0.5*3/4, label = 5, style={color = rwth-magenta}, fontcolor = white, size = 0.3 ]{t12}
			\Vertex[x=2.5, y=-0.9*3/4, label = 4, style={color = rwth-magenta}, fontcolor = white, size = 0.3 ]{t13}
			\Vertex[x=3.5, y=-0.9*3/4, label = 5, style={color = rwth-magenta}, fontcolor = white, size = 0.3 ]{t14}
			\Edge[Direct](t1)(t2)
			\Edge[Direct](t3)(t4)
			\Edge[Direct](t5)(t6)
			\Edge[Direct](t7)(t8)
			\Edge[Direct](t9)(t10)
			\Edge[Direct](t11)(t12)
			\Edge[Direct](t13)(t14)

			\draw[->, line width = 1.0] (2.1, 1.9*3/4)--(2.1, -1.5*3/4);
			\draw[->, line width = 1.0] (2.1, 1.9*3/4)--(3.85, 1.9*3/4);
			\node[ anchor=south ] at ( 3., 1.9*3/4 ) { \small orientation };
			\node[ anchor = south, rotate = 90 ] at (2.1, 0.2*3/4 ) { \small ordering };

		\end{scope}

		\begin{scope}[shift = {(2,-0.5)}]
			\node[] at ( 3.25, -2.0 ) { \( B_2 \textcolor{rwth-blue}{[1, 2, 3 ]} = \overbrace{\textcolor{rwth-red}{(+1)} [1, 2]}^{\substack{\text{1st in}\\\text{order}}} + (-1) [1, 3] + (+1) [2, 3] \) };
			\node[] at ( 3.25, -3.4 ) { \( B_2 \textcolor{rwth-blue}{[1, 3, 4 ]} = \underbrace{\textcolor{rwth-red}{(+1)} [1, 3]}_{\substack{\text{1st in}\\\text{order}}} + (-1) [1, 4] + (+1) [3, 4] \) };
		\end{scope}
	\end{tikzpicture}
	\caption{ Example of a simplicial complex with ordering and orientation: nodes from \( \V 0 \) in magenta, edges from \( \V 1 \) in black, and triangles from \( \V 2 \) in blue. Orientation of edges and triangles is shown by arrows; the action of the \( B_2 \) operator is exemplified for both triangles. Adapted from \cite{savostianov2024cholesky}. \label{fig:orientation}}
\end{figure}

To formally define such a map, we make use of the lexicographical ordering (orientation) with which we endowed the simplices.
To this end, let us consider the linear spaces \( \mc C_k\) of formal sums of simplices in \( \V k \); i.e., \( \mc C_0 \) is the space composed of nodes, \( \mc C_1 \) the space composed of (ordered) edges, and so on.
Note that each such space is isomorphic to $\ds R^{m_k}$, \( \mc C_k \cong \ds R^{m_k} \).
In the following, we will also use the alternative (dual) viewpoint of considering elements of \( \mc C_k \) as signals defined on the simplices in \( \V k \), i.e., functions \( f : \V k \mapsto \ds R\) which may be thought of as ``flows'' supported on the simplices in \( \V k \).
We now define the boundary map \( B_k\) on \( \sigma = [ v_{i_0}, v_{i_1}, \dots v_{i_k} ] \in \V k \) as the following linear map that sends each $k$-simplex to its boundary simplexes via an alternating sum:
\begin{equation*}
	B_k : \mc C_k \mapsto \mc C_{k-1}, \quad B_k \sigma = \sum_{j=0}^{k} (-1)^j \sigma_{\bar j }
\end{equation*}
where \( \sigma_{\bar j } \) denotes the face of \( \sigma \) that does not include \( v_{i_j}\).
Given our fixed orientation for each simplex (as induced by the lexicographical ordering), we can use the ordered simplices in \( \V k \) and \( \V {k-1}\) as a canonical bases for \( \mc C_k\) and \( \mc C_{k-1}\), respectively.
In this basis, we can represent the boundary operators as matrices.
For simplicity, from now on, we will thus use the symbol \( B_k \) to denote the matrix representation of the boundary operator. An example for a simplicial complex and the action of the boundary operator is provided in \Cref{fig:orientation}.

Importantly, the alternating sum in the definition of boundary operators \( B_k \) upholds the fundamental lemma of homology (``a boundary of a boundary is zero''), \( B_k B_{ k + 1 } = 0 \).
This gives rise to the \emph{Hodge decompositon}:
\begin{equation}
	\label{eq:hodge_decomposition}
	\ds R^{m_k} = \lefteqn{\overbrace{\phantom{\im B_k^\top \oplus  \ker \left( B_k^\top B_k + B_{k+1} B_{k+1}^\top \right)}}^{\ker B_{k+1}^\top}} \im B_k^\top \oplus
	\underbrace{\ker \left( B_k^\top B_k + B_{k+1} B_{k+1}^\top \right) \oplus  \im B_{k+1}}_{\ker B_k}.
\end{equation}

\begin{definition}[Hodge Laplacian operator]
	The operator \( L_k = B_k^\top B_k + B_{k+1} B_{k+1}^\top \) is known as the \emph{Hodge} or \emph{higher-order} Laplacian operator and has the following properties:
	\begin{enumerate}
		\item the elements of \( \ker L_k \) in the Hodge decomposition correspond to the \( k \)-dimensional holes in \( \mc K \) (connected components for \( k = 0 \), 1-dimensional holes for \( k = 1\), and so on);
		\item the first term \( \Ld k = B_k^\top B_k \) is known as the \emph{down-Laplacian} and describes the relation between \( k\)- and \( (k-1)\)-simplices. For \( k = 1 \), \( \im \Ld k = \im B_k^\top \) contains so-called gradient flows on the edges;
		\item the second term \( \Lu k = B_{k+1} B_{k+1}^\top \) is known as the \emph{up-Laplacian} and describes the relation between \( k\)- and \( (k+1)\)-simplices. For \( k = 1 \), \( \im \Lu k = \im B_{k+1} \) contains so-called curl flows.
	\end{enumerate}
\end{definition}


Note that boundary and Laplacian operators can be generalized to the weighted case.
Let \( W_k \) be a diagonal matrix such that its \(i\)-th diagonal entry contains the weight of the \(i\)-th simplex in \( \V k \), \( [W_k]_{ii} = \sqrt{w_k(\sigma_i)}\).
Then,  \(W_{k-1}^{-1} B_k W_k \) provides a weighted version of $B_k$ that preserves all the fundamental topological features of $B_k$, and the same is true for the corresponding weighted higher-order Laplacian~\cite{guglielmi2023quantifying}.

By definition, the up-Laplacian \( \Lu{k} \) describes the relationship between simplicies in \( \V{k} \) and \( \V{k+1} \).
At the same time, its spectrum encodes information about the overall topology of the simplicial complex.
As a result, the (spectrum of the) operator \( \Lu{k} \) is an appropriate object to quantify the closeness between two simplicial complexes for the task of \( k \)-sparsification, as we describe below.

\section{Sparsification of Simplicial Complexes}
\label{sec:sparse}

We consider the \( k\)-sparsification of a weighted simplicial complex \( \mc K \).
While we will assume that \( W_{k-1} = I \) in the following, for ease of exposition, every statement below holds (or can be readily generalized) to the general case.
To simplify notation, we will further omit the index \( k \) sometimes.
The task of spectral sparsification we are concerned with can be formalized as follows.
\begin{problem}\label{prob:1}
For a given weighted simplicial complex \( \mc K \) and a sensitivity level \( \eps > 0 \), find a simplicial complex \( \mc L \) such that
\begin{enumerate}
	\item \( \V i = \mc V_i(\mc L)\) for all \( i = 0,\dots,k\), i.e., all simplices of order up to \( k \) are preserved;
	\item \( \mc V_{k+1}(\mc L) \subset \V {k+1} \) with \( m_{k+1}(\mc L) \ll m_{k+1}(\mc K) \), i.e., the number of \( (k+1) \)-simplices is significantly reduced;
	\item \( \Lu k (\mc L ) \underset{\eps}{\approx} \Lu k (\mc K) \), i.e., the up-Laplacians of \( \mc K \) and \( \mc L \) are spectrally \( \eps \)-close.
\end{enumerate}
\end{problem}

In the framework of \Cref{prob:1}, the sparsifier \( \mc L \) is obtained from the original complex \( \mc K \) by keeping all simplices of lower orders as is and subsampling \( (k+1) \)-simplices from \( \V {k+1}\) into \( \mc V_{k+1} (\mc L )\).
Following the seminal work of Spielman and Srivastava~\cite{spielman2008graph, spielman2011spectral}, and its generalization to simplicial complexes~\cite{osting2017spectral}, one can show that such a sparsifier can be obtained by a randomized sampling of \( (k+1) \)-simplices according to a specific probability measure.

More precisely, we sample a chosen number, \( q(m_k)\), of \( (k+1) \)-simplices from the original set \( \V {k+1} \) according to a probability measure \( \b p \) over \( \V {k+1}\).
To obtain a good sparsifier fo the up-Laplacian \( \Lu k(\mc L) \) for sufficiently large \( q(m_k) \), such sampling has to be \textit{unbiased}, i.e., under our sampling scheme we have to recover the original up-Laplacian \( \Lu k (\mc K) \) \emph{in expectation}.
This property may be achieved by (a) scaling the weights of sampled simplices as \( w_{k+1} (\sigma) \mapsto \frac{w_{k+1}(\sigma)}{q(m_k)\b p(\sigma)}\) and (b) accumulating the weights of repeatedly sampled simplices.
Both of these measured combined guarantee that each simplex \( \sigma \in \V {k+1} \) is sampled with its original weight \( w_{k+1}(\sigma) \) in expectation, i.e., we obtain an unbiased sampling scheme.

In the unbiased case, the probability of large deviations between randomly sampled up-Laplacian \( \Lu k (\mc L )\) and the original up-Laplacian \( \Lu k (\mc K)\) can be bound by the concentration inequalities with the probability measure \( \b p \) affecting such bound (see, f.i. Rudelson's theorem, \cite{batson2013spectral, rudelson1999random}).
As a result, the optimal choice of the probability measure \( \b p \) allows for faster concentration, smaller sample size \( q(m_k)\) and sparser resulting complex \( \mc L \).
The work \cite{osting2017spectral} suggested to choose the probability measure proportional to the generalized effective resistance of the simplices:
\begin{equation*}
	\b p \sim W^2_{k+1} \b r, \qquad \text{ where } \qquad  \b r = \diag \left( B_{k+1}^\top ( \Lu k )^\dagger B_{k+1} \right).
\end{equation*}
Under this sampling scheme, a log-linear number \( q(m_k) = O( m_k \log (m_k/\eps^2)) \) of sampled simplices suffices to obtain a spectral sparsifier  \( \Lu k (\mc L ) \underset{\eps}{\approx} \Lu k (\mc K) \). We summarize the steps of the random spectral sparsification in \Cref{alg:sparse} and provide a theoretical guarantee in \Cref{thm:sparsify} below.

\begin{algorithm}[!t]
	\caption{Pseudocode for spectral sparsification}
	\label{alg:sparse}
	\begin{algorithmic}[1]
		\REQUIRE up-Laplacian \( \Lu k \), boundary map \( B_{k+1} \), weight function \( w_{k+1} (\cdot)\), number of sampled simplices \( q(m_k)\)
		\STATE \( \b r \gets \diag \left( B_{k+1}^\top ( \Lu k )^\dagger B_{k+1} \right)  \) \hfill \textcolor{rwth-black-50}{\COMMENT{form the GER vector}}
		\STATE \( \b p \gets W^2_{k+1} \b r / \| W^2_{k+1} \b r \|_1 \)\hfill \textcolor{rwth-black-50}{\COMMENT{form and normalize probability measure}}
		\STATE \( T \gets \texttt{sample}( \V {k+1}, q(m_k), \text{replace} = \texttt{true}) \) \hfill \textcolor{rwth-black-50}{\COMMENT{sample \( (k+1) \)-simplices} }
		\STATE \( \mc V_{k+1} ( \mc L ) \gets \emptyset \), \( \widetilde {W}^2_{k+1} \gets 0 \) \hfill \textcolor{rwth-black-50}{\COMMENT{initialize sparsifier with new weights} }
		\FOR { \( \sigma \in T \) }
		\STATE \( \mc V_{k+1} ( \mc L ) \gets \mc V_{k+1} (\mc L ) \cup \{ \sigma \} \) \hfill \textcolor{rwth-black-50}{\COMMENT{add \( \sigma \) to the sparsifier}}
		\STATE \( \widetilde {W}^2_{k+1} [ \sigma; \, \sigma ] \gets  \widetilde {W}^2_{k+1} [ \sigma;\, \sigma ] + \frac{w_{k+1}(\sigma)}{q(m_k)\b p(\sigma)} \) \hfill \textcolor{rwth-black-50}{\COMMENT{accumulate bias-corrected weights}}
		\ENDFOR
		\STATE \( \mc L \gets { \displaystyle \cup_{i = 0}^k \V k} \cup \mc V_{k+1} (\mc L )  \)
		\RETURN sparsifier \( \mc L \), new weights \( \widetilde{W}^2_{k+1} \)
	\end{algorithmic}
\end{algorithm}

\begin{theorem}[Simplicial Sparsification, \cite{spielman2008graph, osting2017spectral}] \label{thm:sparsify}
	For any \( \eps > \frac{1}{\sqrt{m_k}} > 0 \) and a given simplicial complex \( \mc K \), let complex \( \mc L \) be a sparsifier sampled according to \Cref{alg:sparse} with  \( q ( m_k ) \ge 9 C^2 m_{k} \log ( m_{k} / \eps^2  )\), for some absolute constant  \( C>0 \).
	Then, with probability at least \(1/2\), the up-Laplacian of \( \mc L \) is spectrally \(\eps\)-close to the up-Laplacian of $\mc K$, i.e., we have  \( \Lu k (\mc L ) \underset{\eps}{\approx} \Lu k (\mc K) \).
\end{theorem}

The computational bottleneck in \Cref{alg:sparse} is finding the generalized effective resistance vector \( \b r \), as it amounts to computing the pseudo-inverse \( ( \Lu k )^\dagger \), which has computational complexity \( \mathcal{O}( m_k^3 + 2 k m_k m_{k+1} ) \).
This poses the central problem of the current work:

\begin{problem}
Find a computationally efficient and arbitrarily close approximation of the generalized effective resistance vector \( \b r \) for a given weighted simplicial complex \( \mc K \) and simplices of fixed \( k \)-th order.
\end{problem}

In the rest of the paper, we formulate and discuss a novel method for approximating the generalized effective resistance vector \( \b r \), using efficiently computable spectral densities of the up-Laplacian operators \( \Lu k \).
Note that the unbiased random sparsification process defined above admits any probability measure \( \b p \), and hence, any approximation of the effective resistance vector $
	\b r$ (\Cref{alg:sparse}) can be used for the sparsification.
However, non-optimal choices of the measure \( \b p \) will require far larger number of sampled simplices \( q(m_k) \) to achieve similar spectra for \( \Lu k (\mc K )\) and \( \Lu k (\mc L )\).
Hence a good approximation of \( \b r \) is critical to obtain a good sparsifier with a small number of sampled simplices.
Below we demonstrate the low sensitivity of the sparsification to the moderate perturbations of the resistance-based probability measure \( \b p \sim W^2_{k+1} \b r \).

\begin{remark}
	Since \( \b p \) is a probability measure over \( \V {k+1} \), it is natural to measure its perturbations in terms of \( \frac{1}{m_{k+1}} \), i.e.,
	\(
	\left| \b p (\sigma_i) - \b p^{(\delta)}(\sigma_i) \right|< \frac{\delta}{m_{k+1}},
	\)
	since the size of the perturbation is meaningful only in relation to the actual support of the measure.
	As shown in \Cref{fig:perturb_measure}, random perturbations of \( \b p \) can, on average, slow down the convergence of the approximately sampled complex \( \mc{L}^{(\delta)} \) to the original simplicial complex \( \mc{K} \) in terms of the number of sampled simplices.
	However, even for moderately high values of \( \delta \), such as \( \delta = 0.5 \), the convergence rate remains largely unaffected.
\end{remark}

\begin{figure}[tb!]
	\centering
	\vspace{-10pt}
	\includegraphics[width = 0.8\columnwidth]{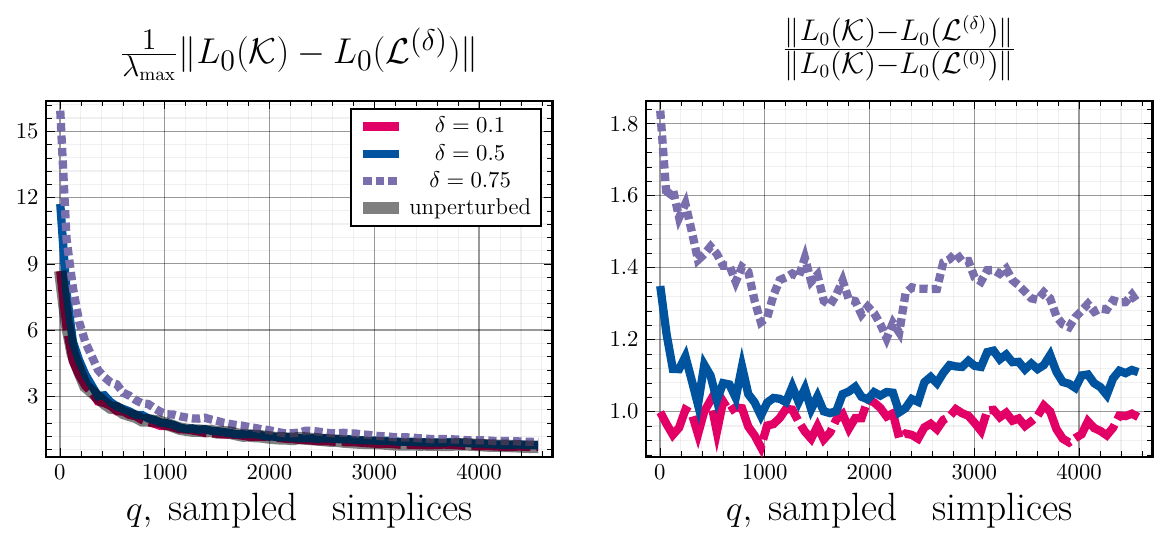}
	\caption{ Convergence of the sampled simplicial complex \( \mc L \) to the original complex \( \mc K \) at \(0\)-order in terms of the spectrum of \( L_0 \) Laplacian operator. Left pane: convergence rate vs the number of sampled edges \( q \) for various perturbed measures \( \b p^{(\delta)}\). Right pane: convergence rate for chosen values of \( \delta \) in relation to the unperturbed sparsifier. \( m_0 = 100 \), \( m_1 = \frac{m_0(m_0-1)}{3}\). All curves are averaged over \( 25 \) random perturbations for VR-complex (see \Cref{sec:benchmark}). \label{fig:perturb_measure}
	}
\end{figure}

\section{Sparsification measure via Kernel Ignoring Decomposition of
  Local Densities of States}
\label{sec:KID}

In this section we show how the effective resistance vector\( \b r \) can be efficiently approximated via functional descriptors of the up-Laplacian spectrum known as spectral densities or densities of states \( \mu_k(\lambda \mid L_k)\), \cite{benson2016higher}.
Subsequently, we propose a novel method for computing \( \mu_k(\lambda \mid L_k)\).

\subsection{Density of States}

\begin{definition}[Density of States]
	For a given symmetric matrix \( A = Q \Lambda Q^\top \) with \( Q^\top Q = I \) and diagonal \( \Lambda = \diag \left( \lambda_1, \dots \lambda_{n} \right) \), the \emph{spectral density} or \emph{density of states} is defined as
	\begin{equation}
		\mu( \lambda \mid A ) = \frac{1}{n} \sum_{i=1}^{n} \bm\delta \left( \lambda - \lambda_i \right)
	\end{equation}
	where \( \bm\delta(\lambda )\) is a Dirac delta function. Additionally, let \( \b q_i \) be a corresponding unit eigenvector of \( A \) (such that \( A \b q_i = \lambda_i \b q_i \) and \( Q = \left( \b q_1 \mid \b q_2 \mid \dots \mid \b q_n \right)\)).
	Then one can define a set of \emph{local densities of states}:
	\begin{equation}
		\mu_j ( \lambda \mid A ) = \sum_{i=1}^{n} \left| \b e_j^\top \b q_i \right|^2 \bm\delta \left( \lambda - \lambda_i \right)
		\qquad \text{ for } \qquad j = 1, \dots n,
	\end{equation}
	with \( \b e_j \) being the $j$-th canonical basis vector.
\end{definition}
Here, the density function \( \mu( \lambda \mid A ) \) contains the overall spectrum of the operator \( A \) while the family of local densities \( \mu_j ( \lambda \mid A )\) describes the contribution of the simplex \( \sigma_j \in \V k \) to the spectral information.

Finally, one should note that by definition, \( \mu(\lambda \mid A) \) and \( \mu_j (\lambda \mid A) \) are generalized functions.
Hence, the quality of their computation is difficult to assess directly. To this end, one can instead consider their histogram representations:
\begin{equation*}
	h_i = \int_{x_i}^{x_i+\Delta_x} \mu( \lambda \mid A ) d\lambda, \qquad h^{(j)}_i = \int_{x_i}^{x_i+\Delta_x} \mu_j( \lambda \mid A ) d\lambda
\end{equation*}
which correspond to the discretized output of the convolution of spectral densities with a mollifier \( K_{\Delta h} \), i.e., \( h_j(\lambda) =  \left[ \mu_j(\lambda\mid A ) \ast K_{\Delta h }  \right] = \int_{\ds R } \frac{1}{\Delta h} K \left( \frac{ \lambda - \eta }{\Delta h } \right)\mu_j(\eta \mid A) d\eta\), where \( K \) is a smooth approximation to identity.

With the notion of spectral density in place, we can now obtain the following reformulation of the generalized effective resistance vector \( \b r \)  and its computation:
\begin{theorem}[Effective resistance through Local Densities of States]\label{thm:GER_DOS}
	For a given simplicial complex \( \mc K \), with \(k\)-th order up-Laplacian \( \Lu k = B_{k+1}^{} W_{k+1}^2 B_{k+1}^\top\), the generalized effective resistance \( \b r \) can be computed through a family of local densities of states \( \{ \mu_i(\lambda \mid \Ld {k+1 }) \} \) of the $k+1$ down-Laplacian as follows:
	\begin{equation*}
		\b r_i = \int_{\ds R \backslash \{ 0 \} } \mu_i ( \lambda \mid \Ld {k+1} ) d\lambda
	\end{equation*}
\end{theorem}
{\begin{proof}
	Let \( B_{k+1} W_{k+1} = U S V^\top\) where \( S \) is diagonal and invertible and both \( U \) and \( V \) are orthogonal.
	Stated differently, $USV^\top$ is a truncated singular value decomposition of the matrix \( B_{k+1} W_{k+1} \) with eliminated kernel.
	Then:
	\begin{equation}
		( \Lu k )^\dagger = \left( B_{k+1} W_{k+1}^2 B_{k+1}^\top \right)^\dagger = \left( U S^2 U^\top \right)^\dagger = U S^{-2} U^\top
	\end{equation}
	It follows that
	\begin{equation*}
		\begin{aligned}
			\b r & = \diag \left( W_{k+1} B_{k+1}^\top ( \Lu k )^\dagger B_{k+1} W_{k+1} \right)   = \diag \left(  V S U^\top U S^{-2} U^\top U S V^\top \right) = \diag \left(  V V^\top \right)
		\end{aligned}
	\end{equation*}
	As a result we can express \( \b r_i = \| V_{ i \cdot } \|^2 = \sum_j | v_{ij} |^2  \).
	Hence, the \(i\)-th entry of the generalized resistance vector is defined by the sum of squares of the $i$th entries of the eigenvectors \( \b v_j \) associated to nonzero eigenvalues of the down-Laplacian matrix \( \Ld {k+1} = W_{k+1} B_{k+1}^\top B_{k+1} W_{k+1} \).

	Note that
	\begin{equation}
		\mu_i ( \lambda \mid \Ld {k+1} ) = \sum_{j=1}^{m_{k+1}} \left| \b e_i^\top \b q_j \right|^2 \bm\delta \left( \lambda - \lambda_j \right)  = \sum_{j=1}^{m_{k+1}} \left| q_{ij} \right|^2 \bm\delta \left( \lambda - \lambda_j \right).
	\end{equation}
	Moreover, we have $\Ld {k+1} = V S^2 V^\top$ and \( Q = V \) (up to the zero eigenpairs). It thus follows that
	\begin{equation}
		\begin{aligned}
			\b r_i & = \| V_{ i \cdot } \|^2 = \sum_j | v_{ij} |^2 = \int_{ \ds R \backslash \{ 0 \} }  \sum_{j=1}^{m_{k+1}} \left| q_{ij} \right|^2 \bm\delta \left( \lambda - \lambda_j \right)  d\lambda = \int_{ \ds R \backslash \{ 0 \} } \mu_i ( \lambda \mid \Ld {k+1} ) d\lambda
		\end{aligned}
	\end{equation}
\end{proof}}

\Cref{thm:GER_DOS} implies that it is sufficient to obtain the family \( \{ \mu_i(\lambda \mid \Ld {k+1 }) \} \) for the next down-Laplacian \( \Ld {k+1 } \) in order to compute the sparsifying probability measure at the level of \(k\)-simplices.
At the same time, by its definition, the spectral density \( \left\{ \mu_i(\lambda \mid A ) \right\}\) requires the complete spectral information of the original operator \( A \) and, hence, is not immediately computationally beneficial.
To avoid this computational overhead, we leverage the functional nature of the local densities of states and obtain an efficient approximation of \( \{ \mu_i(\lambda \mid \Ld {k+1 }) \} \) via truncated polynomial expansion.

\subsection{Efficient approximation of Local Densities of States}
Fast approximations of spectral densities are typically based on Kernel Polynomial Methods (KPM), \cite{weisse2006kernel}, that schematically operate as follows:
\begin{enumerate}[leftmargin=*]
	\item Shift the operator \( A \mapsto H \) so that \( \sigma( H ) \subseteq [a; b] \);
	\item Consider a polynomial basis \( T_m(x)\) on \( [a; b]\), orthogonal with respect to the weight function \( w(x)\).
	      Then the local densities of states can be decomposed as
	      \[ \mu_j(\lambda \mid H ) = \sum_{m=0}^\infty d_{mj} w( \lambda )T_m(\lambda) \]
	      where the coefficients \( d_{mj} \) are known as \emph{moments}. In practice, one is interested in a truncated decomposition of the form \( \widehat \mu_j(\lambda \mid H ) = \sum_{m=0}^M d_{mj} w(x)T_m(x) \) for some finite \(M\); 
	\item To determine the values of \( d_{mj} \) it can be leveraged that $d_{mj}$ are functions of \( T_m(H)\) and can be efficiently sampled via Monte-Carlo methods.
	      In this contexts, a three-term recurrence of orthogonal polynomial bases can be exploited~\cite{benson2016higher} -- we review this step in more detail below.
\end{enumerate}

We point out that a typical choice for the shift interval is \( [-1, 1]\).
The polynomial basis is typically chosen to correspond to Chebyshev polynomials of the first kind.

A fundamental limitation of KPM is the polynomial nature of the decomposition, which may require a large number of moments \( M \) for \( \widehat \mu_j(\lambda \mid H )\) to produce an accurate approximation for ``pathologic'' functions that are far from being polynomials.
In the case of the (local) densities of states, a particularly challenging setting is associated with eigenvalues of high multiplicity, which result in dominating ``spikes'' in the histogram representations of the spectral densities.
Following KPM, one would have to approximate an outlier with a polynomial function.
In the case of the classical graph Laplacian \( L_0 \) and the adjacency matrix, these spikes may be caused by over-represented motifs in the graph \cite{benson2016higher}.
However, in this setting, one knows the closed form of the corresponding eigenspace, and thus, the over-represented eigenvalues can be explicitly filtered out.
For the general case of up- and down-Laplacians \( \Lu k \) and \( \Ld k \) of order \( k>0 \), the eigenspaces with high multiplicity are unavoidable due to the Hodge decomposition, \eqref{eq:hodge_decomposition}: indeed, since \( \im B_k^\top \subseteq \ker B_{k+1}^\top = \ker \Lu k \) and \( \im B_{k+1} \subseteq \ker B_k = \ker \Ld k \), both operators have large kernels which are detrimental to KPM approximation.
Moreover, the kernel's bases depend on the topology of the simplicial complex and cannot be easily estimated.

Below, we thus propose a novel modified method for approximating \( \{ \mu_i(\lambda \mid \Ld {k+1}) \} \) that intentionally avoids the spike in the kernel of \( \Ld {k+1}\), with all the necessary definitions.

\subsection{Kernel-ignoring Decomposition}
As discussed above, the quality of the KPM approximation of LDoS \( \{ \mu_j( \lambda \mid \Ld k ) \} \) suffers from the large null space of the operator.
However, \Cref{thm:GER_DOS} suggests that the target resistance vector \( \b r \) ignores the values of \( \mu_j( \lambda \mid \Ld {k+1} )\) associated with the kernel since the region of integration excludes \( 0 \).
Note that, since \( \mu_j( \lambda \mid \Ld {k+1} ) \approx \sum_{m =0}^M d_{mj} w(\lambda) T_m(\lambda)  \) is a functional decomposition, the approximation error associated with the spike at \( \lambda = 0 \) will not be localized but will spread across the whole domain.
For that reason, we suggest a modified shifting technique that leads to a decomposition that ignores the singular value of \( \lambda \) associated with the operator's null space.

In prior works, the choice \( H = \frac{2}{\lambda_{\max{}}} \Ld {k+1} - I \) was considered, where \( \lambda_{\max{}}\) denotes the largest eigenvalue of \( \Ld {k+1}\).
Hence, the spectrum is bounded in an interval \( \sigma(H) \subseteq [-1; 1]\) and the null eigenvectors of \( \Ld {k+1}\) are shifted to \( -1 \in \sigma(H)\).
In this work, we instead select \( H = \frac{1}{ \lambda_{\max{}} } \Ld {k+1}\) and define a symmetrized version of the local densities of states as:
\begin{equation}
	\tilde \mu_j( \lambda \mid H ) = \begin{cases}
		\mu_j( \lambda \mid H ),  & \text{ if } \lambda \in (0, 1]  \\
		0,                        & \text{ if } \lambda = 0         \\
		-\mu_j(-\lambda \mid H ), & \text{ if } \lambda \in [-1, 0)
	\end{cases}
\end{equation}
Note that the support of the symmetrized \(\tilde \mu_j( \lambda \mid H )\) still falls within \( [-1; 1 ]\), and the spike associated with the value \( \lambda = 0 \) is by design tied to \( 0 \).

The remainder of the approximation approach can now be adopted from the KPM method.
Let us assume that \( T_m(x)\) are Chebyhev polynomials of the first kind. Specifically,
\begin{equation}
	T_0(x) = 1, \; T_1(x) = x, \qquad T_{m+1}(x) = 2x T_m(x) - T_{m-1} (x)
\end{equation}
forming an orthonormal basis on \( [-1, 1 ]\) with respect to the scalar product defined by the weight function \(w(x) = 2/((1+\delta_{0m}) \pi \sqrt{1-x^2})\).

We decompose the symmetrized local density of states \( \tilde \mu_j (\lambda \mid H )\) as:
\begin{equation}
	\tilde \mu_j( \lambda \mid H ) = \sum_{m=0}^\infty d_{mj} w(\lambda)T_m(\lambda),
\end{equation}
Since \( \tilde \mu_j (\lambda \mid H )\) is odd by design, its decomposition should contain only odd Chebyshev polynomials.
Moreover, the remaining \( d_{mj} \) for odd \( m \) can be explicitly expressed through the entries of \( T_m(H)\) as follows:

\begin{lemma}
	Let \( T_m(x)\) be Chybeshev polynomials of the first kind and \(\tilde \mu_j (\lambda \mid H )\) be a symmetrized local density of states defined above. Then its moments \( d_{mj}\) are given by:
	\begin{equation}
		d_{mj} = \begin{cases}
			0,                            & \quad \text{if } m \text{ is even} \\
			2\; \sum_{\mathclap{\substack{ \vspace{9pt}                        \\ \lambda_i \in \sigma(H)\backslash \{ 0 \}}}} \;\; \left| \b e_j^\top \b q_i \right|^2  T_m(\lambda_i) =
			2 \left[ T_m(H) \right]_{jj}, & \quad  \text{if } m \text{ is odd} \\
		\end{cases}
	\end{equation}
	thus, the $m$-th vector of moments (for odd \(m\)) can be expressed as $d_{m\bullet} = 2 \diag( T_m(H) )$.
\end{lemma}

\begin{proof}
	Let \( \mu_j ( x \mid A )  \) be the LDoS for a general symmetric matrix $A$ with spectra decomposition \( A = Q \Lambda Q^\top \).
	Then for an arbitrary polynomial function \( f(x) \), we can compute
	\begin{align}
		\left\langle  \mu_j ( x \mid A ), f(x) \right\rangle & = \int_{-\infty}^{+\infty}  \sum_{i=1}^{n} \left| \b e_j^\top \b q_i \right|^2\bm\delta \left( x - \lambda_i \right) f(x) dx = \\
		                                                     & = \sum_{i=1}^{n} \left| \b e_j^\top \b q_i \right|^2 \int_{-\infty}^{+\infty}  \bm\delta \left( x - \lambda_i \right) f(x) dx  \\
		                                                     & =  \sum_{i=1}^{n} \left| \b e_j^\top \b q_i \right|^2 f(\lambda_i) =  \sum_{i=1}^{n} \left| q_{ji} \right|^2 f(\lambda_i)
	\end{align}
	Since \( f(A) = Q f(\Lambda) Q^\top \), we get
	\begin{equation}
		\left[ f(A) \right]_{jj} = \b e_j^\top Q f(\Lambda) Q^\top \b e_j = \b q_j^\top f(\Lambda) \b q_j = \sum_{i=1}^{n} \left| q_{ji} \right|^2 f(\lambda_i) = \left\langle  \mu_j ( x \mid A ), f(x) \right\rangle
	\end{equation}

	The case for the symmetrized LDoS \( \tilde \mu_j ( x \mid H ) \) is only marginally different:
	\begin{equation*}
		\left\langle  \tilde \mu_j ( x \mid H ), f(x) \right\rangle = \sum_{\substack{i=1   \\ \lambda_i \ne 0 }}^{n} \left| q_{ji} \right|^2 \left( f(\lambda_i) - f(-\lambda_i) \right)
	\end{equation*}
	Due to the orthonormality of \( T_m(x) \), moments \( d_{mj}\) are given by inner products, \( d_{mj} = \left\langle  \tilde \mu_j ( x \mid H ), T_m(x) \right\rangle \).
	For even \( m \), Chebyshev polynomial \( T_m(x) \) is even, hence \( T_m(\lambda_i) -= T_m(-\lambda_i)\) and \( d_{mj} = 0 \).
	In the case of odd \(m\), \( T_m(x) \) is odd itself and \( T_m(0) = 0 \), so, as a result:
	\begin{equation}
		d_{mj} = \sum_{\substack{i=1 \\ \lambda_i \ne 0 }}^{n} \left| q_{ji} \right|^2 \left( T_m(\lambda_i) - T_m(-\lambda_i) \right) =
		2 \sum_{\substack{i=1                                                                 \\ \lambda_i \ne 0 }}^{n} \left| q_{ji} \right|^2 T_m(\lambda_i) = 2 \sum_{i=1}^{n} \left| q_{ji} \right|^2 T_m(\lambda_i)
	\end{equation}
	where the last equality follows from the fact that \( T_m(0) = 0 \).
	Finally, since \( T_m(x)\) are polynomial functions, we get
	\begin{equation}
		d_{mj} = 2 \sum_{i=1}^{n} \left| q_{ji} \right|^2 T_m(\lambda_i) = 2 \left[ T_m(H) \right]_{jj}
	\end{equation}
\end{proof}

Instead of computing the diagonal elements \( T_m(H) \) directly, we can use Monte-Carlo estimations for the diagonal.
Specifically, we will use the fact that for any matrix \( X \):
\begin{equation}
	\mathrm{diag}\, X = \ds E \left[ \b z \odot X \b z \right],
\end{equation}
where \( \b z \) is a vector of i.i.d. random variables with zero mean and unit variance.
In practice, we approximate the expectation via sampling with \( N_z \) samples:
\begin{equation}
	\mathrm{diag}\, X \approx \frac{1}{N_z } ( Z \odot X Z ) \b 1
\end{equation}
where \( \odot \) is the Hadamard (element-wise) product,
\( \b 1 \) is a vector of ones, and \( Z \) is a matrix collecting \(N_z\) copies of \( \b z \) column-wise  \cite{hutchinson1989stochastic, meyer2021hutch++}.

This Monte-Carlo sampling strategy reduces the compuations of \( \diag ( T_m(H))\) to simple \texttt{matvec} operations.
Importantly the calculations can be efficiently updated for the next order of moments \(d_{m+1, \bullet}\), due to the recurrent definition of \( T_m(x)\).
Assume we store the values of \( T_i(H) Z \) for \( i = 0\dots m\).
In order to obtain \(T_{m+1}(H) Z\), we compute \( T_{m+1}(H) Z = 2 H \cdot ( T_{m}(H) Z ) - ( T_{m-1}(H) Z ) \) requiring only one additional \texttt{matvec} operation.
As a result, the computational cost of the approximation is fixed to \( \mc O \left(  N_z M \,\texttt{nnz} (H) \right)\), where \( \mc O \left( \texttt{nnz} (H) \right)\) is the cost of one \texttt{matvec} operation for the operator \( H \). We provide a brief pseudocode for the KID-approximation method in \Cref{alg:KID}.

\begin{algorithm}[H]
	\caption{Pseudocode for KID-approximation method}
	\label{alg:KID}
	\begin{algorithmic}[1]
		\REQUIRE up-Laplacian \( \Lu k \), scaling \( \lambda \), number of moments \( M \), number of Monte-Carlo vectors \( N_z \)
		\STATE \( H \gets \frac{1}{\lambda} \Lu k \)
		\STATE \( Z \gets \texttt{getRandomSigns}(m_{k+1}, N_z)\) \hfill \textcolor{rwth-black-50}{\COMMENT{form MC-vectors}}
		\STATE \( D \gets 0 \) \hfill\textcolor{rwth-black-50}{\COMMENT{initialize matrix of moments \(d_{mk}\)}}
		\STATE \( T_1 \gets Z \), \quad \( T_2 \gets H Z  \)  \hfill\textcolor{rwth-black-50}{\COMMENT{start Chebyshev sequence}}
		\FOR{ \( m = 3 \) to \( M \)}
		\STATE \( T_m \gets 2 H \cdot T_2 - T_1 \)
		\IF{ \( m  \) is \textbf{odd}}
		\STATE \( D[:, m] \gets \texttt{row\_mean} (T_m \odot Z)\) \hfill \textcolor{rwth-black-50}{\COMMENT{KID trick: ignore even moments}}
		\STATE \( T_1 \gets T_2 \), \quad \( T_2 \gets T_m\)
		\ENDIF
		\ENDFOR
		\STATE GER \( \b r \gets \texttt{histogramIntegral}(D, M)\)
		\RETURN \( \b r \)
	\end{algorithmic}
\end{algorithm}

\subsection{Error Propagation and the Choice of Constants}

The computational complexity of the KID method described above cannot be directly compared with the computation of the resistance vector \( \b r \) via the pseudo-inverse of \( \Lu{k} \), because the complexity of the KID methods is parametrized by \( N_z \) and \( M \) rather than the number of simplices \( m_k \).
Nonetheless, one can establish the relation between the method's parameters \( N_z \) and \(M\) and the original simplicial complex through error estimates for the approximation of the local densities of states, as we demonstrate below.

Let us consider the histogram representations of the exact symmetrized densities {\( h^{(j)} = \left[ \tilde \mu_j(\lambda \mid A ) \ast K_{\Delta h }  \right] \)} and its KPM approximation \( \wh h^{(j)}_M = \left[ \left( \sum_{m=0}^M d_{mj} w(\lambda)T_m(\lambda) \right) \ast K_{\Delta h }  \right] \) where the moments \( d_{mj}\) are Monte-Carlo sampled.
Using the estimation bound from \cite[Thm 4.2]{benson2016higher}, we obtain:
\begin{equation}
	\label{eq:err1}
	\ds E \left\| h^{(j)} - \wh h^{(j)}_M \right\|_\infty \le \frac{1}{\Delta h} \left( \frac{6L}{M} + \frac{2 \| K_{\Delta h} \|_\infty}{\pi \sqrt{N_z}} \right)
\end{equation}
where \( L \) is the Lipschitz constant of \(h^{(j)}\).
This result can be further extended to an error estimate for the KID-approximation; we start by showing the following auxiliary fact:
\begin{lemma}
	For a given simplicial complex \( \mc K \) and GER vector \( \b r \), it holds that \( \| \b r \|_1 = m_{k} - \sum_{i=-1}^{k-1} (-1)^{k-1-i} ( m_i - \beta_{i+1} ) \), where \( \beta_k = \dim \ker L_k \) denotes \( k\)-th Betti's number and \( m_{-1} = 0 \).
\end{lemma}
\begin{proof}
	From the proof of \Cref{thm:GER_DOS}, we know that \( \b r_i \ge 0 \), and
	\begin{equation}
		\| \b r \|_1 = \sum_i \b r_i = \sum_{i, j} | v_{ij} |^2
	\end{equation}
	Since each of the right singular vectors \( V_{ \cdot j} \) of \( B_{k+1} W_{k+1}\) has unit length, \( \sum_{i, j} | v_{ij} |^2 = \text{number of columns of } V = \dim \im \Lu k = m_k - \dim \ker \Lu k \)  given that singular vectors \( V_{\cdot j}\) are defined via truncated SVD.
	Due to the discrete Hodge decomposition \eqref{eq:hodge_decomposition}, \( \dim \ker \Lu k = \dim \im B_k^\top + \dim \ker L_k \) and \( \im B_k^\top = \im \Ld k\).
	According to the spectral inheritance principle for Hodge Laplacians \cite{guglielmi2023quantifying}, \( \sigma_+( \Ld k ) = \sigma_+(\Lu {k-1})\) where \( \sigma_+ \) denotes the set of positive eigenvalues of corresponding operators; as a result, \( \dim \im \Ld k = \dim \im \Lu {k-1} \).
	Employing such steps recursively, one obtains  \( \| \b r \|_1 = m_{k} - \sum_{i=-1}^{k-1} (-1)^{k-1-i} ( m_i - \beta_{i+1} ) \).
\end{proof}

We can now formulate the following error bound for the KID method:
\begin{theorem}\label{thm:error}
	For any fixed \( \delta > 0 \), let \( \bf p \) be the exact sparsifying probability measure  for a given sufficiently dense simplicial complex \( \mc K \).
	Let \( \wh{ \bf p }\) be the corresponding KID-approximated sparsifying probability measure.
	If the approximation \( \wh{ \bf p } \) is truncated at \( M \ge 24 L \frac{m_{k+1}}{\delta m_k}\) moments, and we use \( N_z \ge \frac{8 \| K_{\Delta h}\|^2}{\pi^2} \frac{m^2_{k+1}}{\delta^2 m_k^2}\) samples, then
	\begin{equation}
		\| \b p - \wh{ \b p } \|_\infty \le \frac{ \delta }{m_{k+1}}
	\end{equation}
	and the computational complexity of the KID approximation is
	\[
		\mc O \left( \delta^{-3} {m^4_{k+1}}{m_k^{-3}} \right).
	\]
\end{theorem}

\begin{proof}	To show the estimate for the approximation of the sparsifying norm, we consider the estimate
	\begin{equation}
		\ds E \left\| h^{(j)} - \wh h^{(j)}_M \right\|_\infty \le \frac{1}{\Delta h} \left( \frac{6L}{M} + \frac{2 \| K_{\Delta h} \|_\infty}{\pi \sqrt{N_z}} \right)
	\end{equation}
	which translates to the estimate on vector \( \b r\).
	Given each \( \b r_i \ge  0\), the probability measure \( \b p \) is given by \( \b p = \frac{1}{\| \b r \|_1} \b r \).
	As a result, to obtain the bound \( \| \b p - \wh{ \b p} \|_\infty \le \frac{\delta}{m_{k+1}} \), it is sufficient to show:
	\begin{equation}
		\frac{6L}{M} + \frac{2 \| K_{\Delta h} \|_\infty}{\pi \sqrt{N_z}} \le \frac{\delta}{m_{k+1}} \| \b r \|_1.
	\end{equation}
	The bound above is guaranteed for parameters \( M \) and \( N_z \) satisfying:
	\begin{equation}
		\frac{6L}{M} \le \frac{\delta}{2m_{k+1}} \| \b r \|_1 , \qquad \frac{2 \| K_{\Delta h} \|_\infty}{\pi \sqrt{N_z}} \le \frac{\delta}{2m_{k+1}}\| \b r \|_1
	\end{equation}
	Finally, assuming a sufficiently dense (namely, \( \frac{m_k}{2} \ge m_{k-1} + \beta_k \) ) simplicial complex, one gets \( \| \b r \|_1 = m_{k} -\sum_{i=-1}^{k-1} (-1)^{k-1-i} ( m_i - \beta_{i+1} ) \ge m_k - \beta_k - m_{k-1}  \ge \frac{1}{2}m_{k}\), completing the lower bounds on \( M \) and \( N_z\) in the theorem's statement.
\end{proof}

To get a clearer understanding of the overall complexity, recall that \( m_{k+1} = \mc O \left(m_k^{ 1 + \frac{1}{k+1} }\right) \) for a simplicial complex that includes all possible subsets of \( k+2\) nodes as \((k+1)\)-simplices (f.i. completely connected graph with all possible triangles), resulting into the worst-case complexity of \( \mc O \left( \delta^{-3} m_k^{ 1 + \frac{4}{k+1}} \right)\).
In other words, our approximation is not worse than direct computation for \( k = 1 \) (even in the densest case) and is asymptotically linear in \( k \).

\section{Benchmarking}\label{sec:benchmark}

In this section, we numerically investigate the performance of the KID approximation and its computational complexity.
In particular, through our experimental evaluation, we aim to do the following:
\begin{enumerate}[label=\bfseries(\roman*),leftmargin=*]
	\item support the asymptotic estimate for the approximation error \eqref{eq:err1} in terms of the number of moments \( M \) and the number of Monte-Carlo sample vectors \( N_z \);
	\item showcase the computational complexity of the KID approximation using the (scaled) oracle choice for the parameters, \Cref{thm:error};
	\item compare the actual execution time of the approximation to the direct computation for complexes of different sizes and densities.
\end{enumerate}

\paragraph{Vietoris--Rips filtration} \Cref{thm:error} and Equation~\eqref{eq:err1} describe the performance of the developed method in terms of the number of simplices \( m_k \).
To numerically illustrate these behaviours appropriately, we consider a family of arbitrarily large and dense simplicial complexes.
For this reason, we use simplicial complexes induced by a filtration procedure on point clouds.
Formally, we proceed as follows:
\begin{enumerate}[leftmargin=*]
	\item  We consider \( m_0 \) points embedded in $\mathbb R^2$, sampled randomly in two clusters.
	      Specifically, \( \frac{m_0}{2}\) points are sampled from \( \mc N( \b 0, I )\) and \( \frac{m_0}{2}\) points are sampled from \( \mc N( c \b 1, I )\), for some \( c > 0 \).
	\item For a fixed filtration threshold \( \epsilon > 0 \), a simplex \( \sigma = [v_{i_1}, ... v_{i_p} ] \) on these nodes enters the generated complex \( \mc K \) if and only if $d_{\mc M }(v_{i_j}, v_{i_k}) \le \epsilon$ for all pairs $j$ and $k$.
\end{enumerate}
This straightforward filtration is known as  Vietoris--Rips filtration, and the corresponding complex \( \mc K \) as a VR-complex.
An illustrative example is provided in \Cref{fig:example}. In the chosen setup, the value of the filtration parameter \( \epsilon \) naturally governs the density of the generated simplicial complexes of every order, as shown by the right panel in \Cref{fig:example}: larger values of \( \epsilon \) define complexes with a higher number of edges, triangles, tetrahedrons, etc., until every possible simplex is included in \( \mc K \).

\begin{figure}[t]
	\centering
	\includegraphics[width = 0.8\columnwidth]{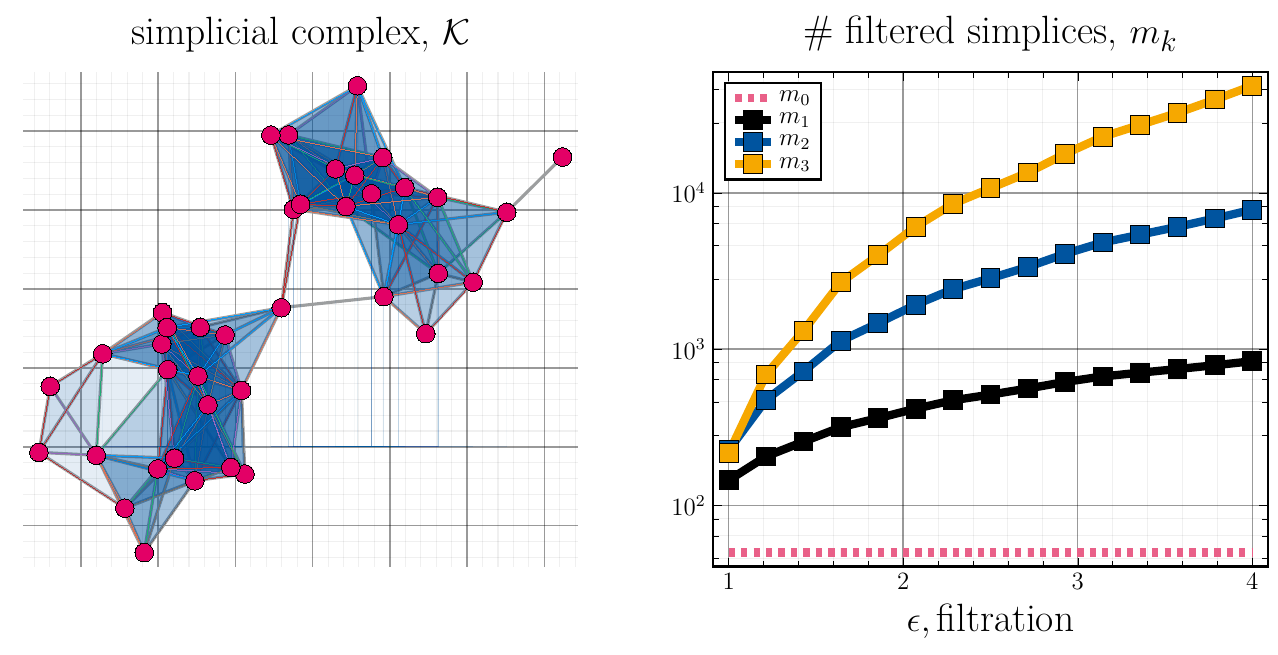}
	\caption{ Example of VR-filtration. Left pane: point cloud with \( m_0 = 40 \) and filtration \( \epsilon = 1.5 \), inter-cluster distance \( c = 3 \). Right pane: dynamics of the number of simplices of different orders for varying filtration parameter \( \epsilon\). \label{fig:example}}
\end{figure}

\paragraph{Parameter choice and computational complexity}
The error estimate from \Cref{eq:err1} suggests that the approximation error for the sparsifying norm \( \b p \) scales as \( M^{-1}\) in terms of the number of moments and as \( N_z^{-1/2}\) in terms of the number of Monte-Carlo vectors (MC-vectors).
To illustrate this behaviour, we fix one of the parameters (\( M \) or \( N_z \)) to their theoretical estimates provided by \Cref{thm:error} and demonstrate the dynamic of the error \( \| \b p - \wh{\b p }\|_\infty \) as the function of the other parameter.
As shown by \Cref{fig:M_Nz}, the overall scaling law coincides with the estimates of \Cref{eq:err1} in the case of \(1 \)-sparsification for \( \Lu 1\) operator.
Note that all experiments are conducted in the a sufficiently dense setting, where \( m_2 \ge m_1 \ln m_1 \).
\begin{figure}[t]
	\centering
	\includegraphics[width = 0.8\columnwidth]{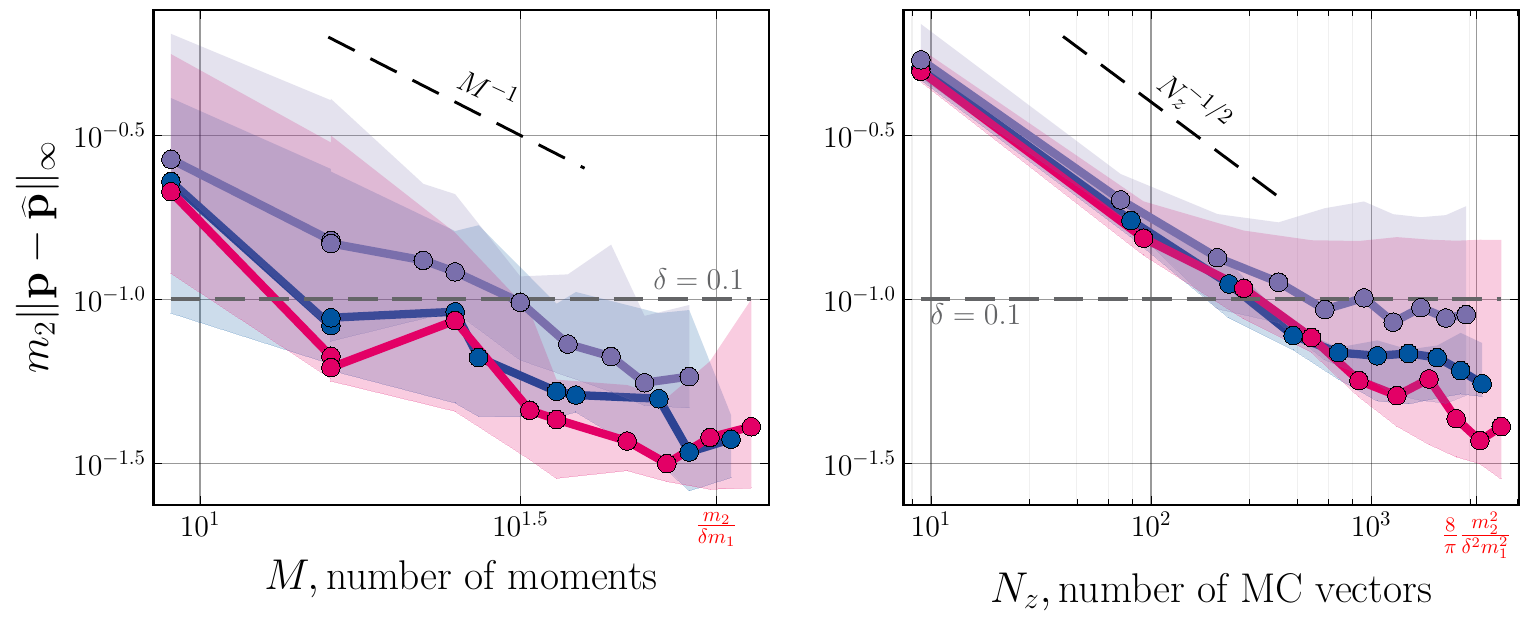}
	\caption{
		Dependence of the approximation error \( \| \b p - \wh{ \b p } \|_\infty \) on the number of moments \( M \) and number of MC vectors \( N_z \). Values are tested up to (scaled) theoretical bounds from \Cref{thm:error} (in red); line colors correspond to varying \( m_0 \) in the point cloud. Left pane: errors vs the number of moments \( M \) with fixed theoretical \( N_z \); right pane: errors vs the number of MC vectors \( N_z \) with fixed theoretical \( M \).  Errors are averaged over several generated VR-complexes; colored areas correspond to the spread of values. \label{fig:M_Nz}
	}
\end{figure}

We explicitly highlight two observations from \Cref{fig:M_Nz}:
(1) larger and denser simplicial complexes tend to exhibit faster convergence in both parameters (especially in the number of moments \( M \));
(2) \Cref{thm:error} provides theoretical (greedy) estimates for \( M \) and \( N_z \) that are sufficient for achieving the target approximation quality \(\delta\) and can be interpreted asymptotically.
Consequently, we may choose scaled (and empirically sufficient) values for these parameters:
\[
	M = \left\lceil \frac{m_{k+1}}{\delta\,m_k} \right\rceil
	\quad\text{and}\quad
	N_z = \left\lceil \frac{1}{10}\,\frac{8}{\pi^2}\,\frac{m_{k+1}^2}{\delta^2\,m_k^2} \right\rceil.
\]

\begin{figure}[t]
	\centering
	\includegraphics[width = 0.8\columnwidth]{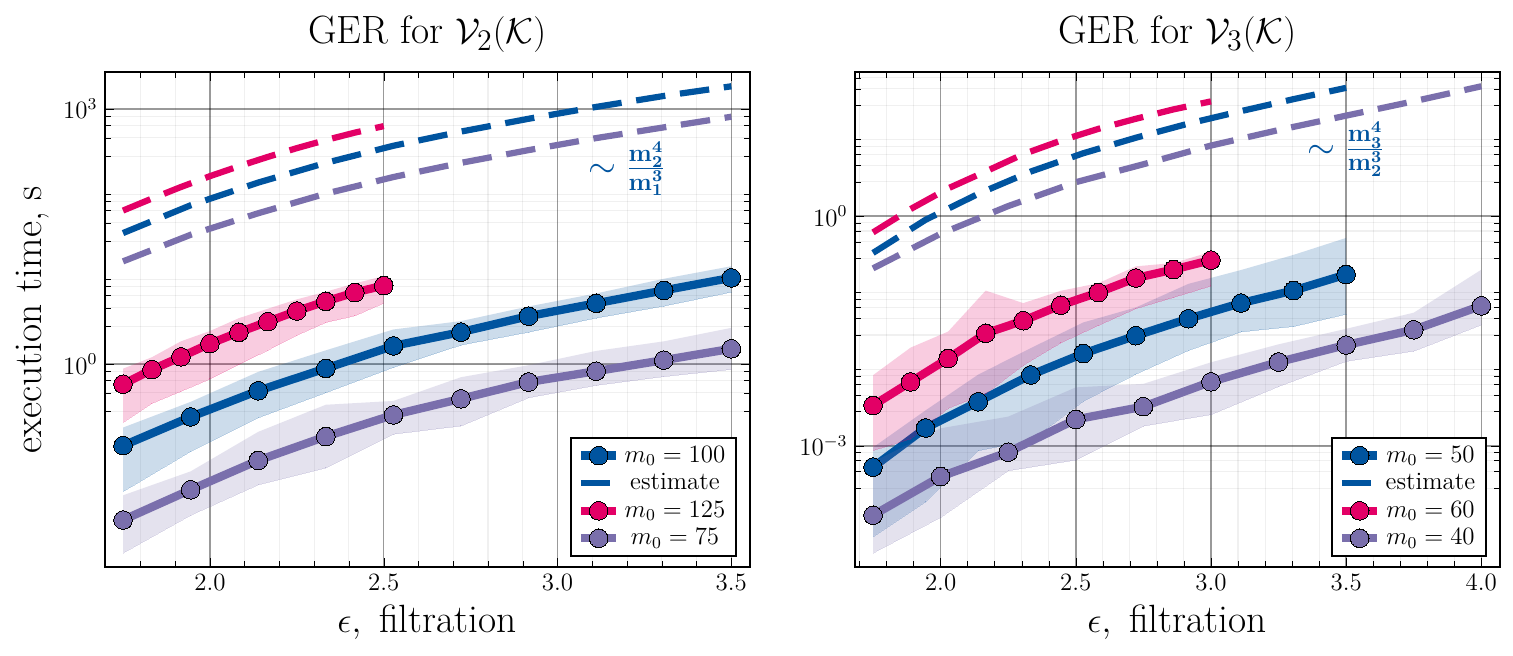}
	\caption{
		Execution time of KID approximation for effective resistance of triangles, \( \V 2\) (left), and tethrahedrons, \( \V 3 \) (right). Line colors correspond to varying \( m_0 \)  in the point cloud; theoretical estimation of the computational complexity is given in dash.
		Execution times are averaged over several generated VR-complexes; colored areas correspond to the spread of values.  \label{fig:times}
	}
\end{figure}

Given this choice of parameters, in \Cref{fig:times} we demonstrate that the complexity estimate
\(\mathcal{O}\!\bigl(\tfrac{m_{k+1}^4}{m_k^3}\bigr)\)
from \Cref{thm:error} aligns with the actual execution time of the KID approximation for varying filtration parameters \(\epsilon\) in the cases of \(1\)- and \(2\)-sparsification of VR-complexes.

\paragraph{Comparison with the direct computation}
Finally, we compare the execution time of the KID approximation with that of the direct computation of the sparsifying measure \( \b{p} \) for \( 1 \)-sparsification, using the approximation parameters mentioned above (see \Cref{fig:comparison}).
Note that although the densest case complexity estimate
\(\mathcal{O}\!\bigl(\delta^{-3}\,m_k^{\,1+\frac{4}{k+1}}\bigr)\)
suggests that the KID method's execution time might be comparable to direct computation, in practice the developed algorithm is significantly faster while still maintaining the target approximation error
\(\|\b{p} - \widehat{\b{p}}\|_\infty \le \frac{\delta}{m_2}\).

\begin{figure}[t]
	\centering
	\includegraphics[width = 0.8\columnwidth]{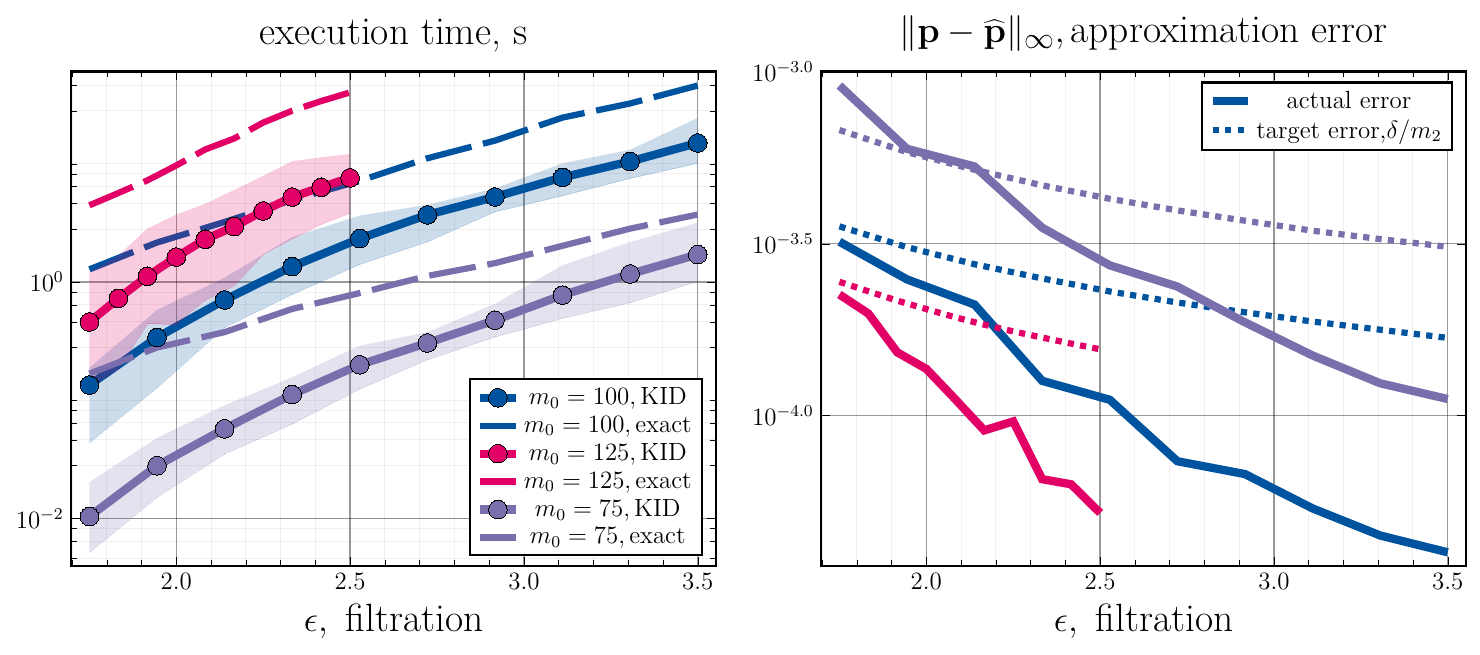}
	\caption{
		Computation time comparison between KID-approximation \( \wh{ \b p} \), solid line, and directly computed sparsifying measure \( \b p \), dashed line (left), and corresponding approximation error \( \| \wh{ \b p} - \b p \|_\infty \) (right). Target approximation error is given in dotted line (right pane);  line colors correspond to varying \( m_0 \)  in the point cloud.
		Execution times are averaged over several generated VR-complexes. \label{fig:comparison}
	}
\end{figure}

Note that the comparison of the performance of the KID method with sparsification methods besides effective resistance sampling is for illustration only, since other methods do not aim to obtain spectrally close sparsifiers.
Nevertheless, we provide a comparison between the KID method, effective resistance sampling, uniform sampling of triangles and \( k\)-Neighbours approach (where one subsamples triangles to guarantee that every edge is up-adjacent to at most \( k \) triangles) in terms of the spectral distance \( \frac{1}{\lambda_{\max})} \| \Lu 1(\mc K) - \Lu 1(\mc L) \|_2\) in \Cref{tab:comparison}
Teh two latter methods are trivially generalized from the case of graphs, \cite{chen2023demystifying}.
As expected, only effective resistance and KID methods are able to achieve spectrally close sparsifiers for moderate numbers of sampled triangles
Note that we suggest uniform sampling with the reweighting technique from \Cref{thm:sparsify}, which implies that the sampled complex \( \mc L \) eventually converges to the original complex \( \mc K \) but requires a larger number of sampled triangles to obtain small spectral distances.
\begin{table}[t]
	\centering
	\caption{Comparison of Effective Resistance and KID approximations with trivial sparsification methods. Results are shown for VR-complexes generated from point clouds with \(m_0 = 50\) and \( \eps = 1.75 \).}
	\label{tab:comparison}
	\begin{tabular}{lccc}
		\hline
		                             & \( 0.1 q(m_k) \) & \( 0.2 q(m_k) \) & \( 0.33 q(m_k) \) \\ \hline
		uniform(unbiased)            & 1.0862           & 0.8634           & 0.5925            \\
		\(k\)-Neighbours             & 1.0              & 1.0              & 0.985             \\
		Effective Resistance         & 0.5824           & 0.2827           & 0.125             \\
		KID(ours), \( \delta = 0.1\) & 0.5996           & 0.3132           & 0.1385            \\
	\end{tabular}
\end{table}

Additionally, the performance of the direct computation of \(\b{p}\) for the largest considered point cloud with \(m_0 = 125\) highlights another important advantage of the KID approximation: reduced memory consumption.
Indeed, whether we use the definition of the resistance vector
\[
	\b{r}
	= \diag \Bigl( B_{k+1}^\top (\Lu{k})^\dagger B_{k+1} \Bigr)
\]
or the reformulation in terms of the right singular vector from \Cref{thm:GER_DOS}, a full SVD of \(\Lu{k}\) is required.
In the case of point clouds with \(m_0 = 125\), denser VR-complexes lead to real-valued matrices of size \(10^4 \times 10^4\), resulting in substantial memory demands for the SVD.
By contrast, the KID approximation avoids this decomposition and restricts the additional memory usage to storing Monte-Carlo matrices \(Z\) and their \texttt{matvecs} of dimension \(m_{k+1} \times N_z\), which is comparatively small.

\subsection{ Real-world data}
In addition to extensive benchmarking on synthetic VR-complexes, we also tested the KID approximation on real-life data provided in Austin Benson's collection of hypergraph data (\url{https://www.cs.cornell.edu/~arb/data/}).
For each hypergraph, we limited consideration to hyperedges of size at most \(10\); then, we performed simplicial closure to obtain simplicial complexes of order \( 2 \).

The results are summarized in \Cref{tab:real_data}: for each dataset, we report the average effective resistance (ER) and KID approximation times and spectral error over \(10\) runs.
Similarly to the case of VR-filtration, the KID approximation is significantly faster and introduces a moderate, but controllable approximation error.
Note, however, that for the largest considered simplicial complex \texttt{vegas-bars}, the memory usage of the direct GER approach prevents its application on a standard laptop while our novel method performs well.

\begin{table}[t]
	\centering
	\caption{Comparison of Effective Resistance (ER) and KID approximation times for various datasets; \( \delta = 0.1 \).}
	\label{tab:real_data}
	\begin{tabular}{lcccccc}
		\hline
		\multicolumn{1}{c}{\textbf{dataset}} & \multicolumn{1}{c}{\(m_1\)} & \multicolumn{1}{c}{\(m_2\)} & \multicolumn{1}{c}{\textbf{ER time}} & \multicolumn{1}{c}{\textbf{ER err}} & \textbf{KID time} & \textbf{KID err} \\ \hline
		email-enron                          & 1604                        & 6578                        & 2.9s                                 & 0.098                               & 0.31s             & 0.1402           \\
		contact-hs                           & 2515                        & 2370                        & 59.81s                               & 0.0962                              & 0.0013s           & 0.1311           \\
		NDC-classes                          & 2658                        & 3778                        & 7.53s                                & 0.0566                              & 0.246s            & 0.5628           \\
		vegas-bars                           & 4078                        & 5099                        & NA                                   & NA                                  & 0.1179s           & 0.1038           \\
		algebra                              & 1934                        & 3191                        & 2.221s                               & 0.0479                              & 0.433s            & 0.0504
	\end{tabular}
\end{table}

\section{ Discussion}
\label{sec:discussion}

In this work we have proposed a fast method of approximating generalized effective resistance vector for simplices of an arbitrary order \( k \) inside simplicial complex \( \mc K \) with the algorithmic complexity \( \mc O \left( \delta^{-3} m_k^{ 1 + \frac{4}{k+1}} \right)\), \Cref{thm:error}, allowing for efficient \( k\)-sparsification of \( \mc K \) through subsampling, \Cref{thm:sparsify}.
Our novel approach is based on the connection between generalized effective resistance vector \( \b r \) and the family of local density of states of the corresponding higher-order down-Laplacian operator \( \Ld {k+1}\), \Cref{thm:GER_DOS}.
We avoid problematic behaviour of the pre-existing KPM approximation methods for LDoS by suggesting a kernel-ignoring decomposition, which de facto decomposes a symmetrized spectral density via odd Chebyshev polynomials.
Our approach numerically follows our theoretical estimates, \Cref{thm:error} and Equation~\eqref{eq:err1}, for the approximation error.
This enables us to choose the number of moments \( M \) and number of Monte Carlo vectors \( N_z \), to control the final approximation error, which is shown to be sufficient at a moderate level for a close-to-efficient sparsification, \Cref{fig:perturb_measure}.
Given the fact that the developed method is only dependent on the upper Hodge Laplacian \( \Lu k \), it can also be directly applied to cell complexes.

Several applications may directly benefit from the proposed method.
Specifically, the introduction of a sparsified complex in label spreading, spectral clustering or generic simplicial complex GNN tasks may sufficiently decrease computational costs, \cite{yang2022simplicial,ebli2019notion}.

Additionally, the cost of trajectory classification as well as landmark detection algorithms for SCs can be directly scaled down by transitioning to a sparser, but spectrally similar model, \cite{grande2024topological}.
Separately, one may notice that the existence of the efficient sparsifier \( \mc L \) effectively bridges the gap between dense complexes and complexes for which one can obtain a preconditioner through a collapsible subcomplex \cite{savostianov2024cholesky}.

We remark that, while the developed method admits graph sparsification as a special case, the computational complexity for \( k = 0 \) would clearly exceed the copmutational complexity of preexisting graph sparsification algorthims. Note that we do not suggest the application of KID-approximations for the case of the  graphs, since it is not the focus of the current work.

Finally, it should be noted that the approximation quality given in \Cref{thm:error} may still be improved, e.g. by accounting for over-represented motifs in the graph (see \cite{dong2019network}) through similar filtration techniques.
Whilst we avoid explicitly defining a generalized motif for simplices of arbitrary order this appears to be a promising future venue of research.

\section*{Acknowledgments}
AS and MTS acknowledge funding by the Deutsche Forschungsgemeinschaft (DFG, German Research Foundation) -- Project number 442047500 through the Collaborative Research Center ``Sparsity and Singular Structures” (SFB 1481). MTS acknowledges funding by the European Union (ERC, HIGH-HOPeS, 101039827). Views and opinions expressed are however those of the author(s) only and do not necessarily reflect those of the European Union or the European Research Council Executive Agency. Neither the European Union nor the granting authority can be held responsible for them.

The research of NG was supported by funds from the Italian MUR (Ministero dell'Universit\'a e della Ricerca) within the PRIN 2022 Project ``Advanced numerical methods for time dependent parametric partial diﬀerential equations with applications'' and the PRO3 joint project entitled ``Calcolo scientifico per le scienze naturali, sociali e applicazioni: sviluppo metodologico e tecnologico''. FT is partially funded by the PRIN-MUR project MOLE code: 2022ZK5ME7 Principal MUR D.D. financing decree n. 20428 of November 6th, 2024, CUP B53C24006410006. FT and NG were supported from MUR-PRO3 grant STANDS and the PRIN-PNRR project FIN4GEO within the European Union's Next Generation EU framework, Mission 4, Component 2, CUP P2022BNB97.
FT and NG are members of INdAM GNCS (Gruppo Nazionale di Calcolo Scientifico).

\bibliographystyle{siamplain}
\bibliography{references}

\end{document}